\documentclass[10pt]{article}
\PassOptionsToPackage{dvipsnames}{xcolor}
\usepackage{custom-style}
\PassOptionsToPackage{compress, sort, numbers}{natbib}
\usepackage[main,final]{neurips/neurips_2026}
\usepackage{hyperref}

\title{Estimating Implicit Regularization in Deep Learning}
\author{%
  Joseph H.~Rudoler
  \quad
  Kevin Tan
  \quad
  Giles Hooker
  \quad
  Konrad P.~Kording \\
  \\
  University of Pennsylvania
}

\usepackage{amsmath,bbm,bm}
\usepackage{amsfonts}
\usepackage{amsthm}
\usepackage{mathtools}

\newtheorem{thm}{Theorem}

\newtheorem{aspt}{Assumption}

\def\1{\bm{1}}

\DeclareMathAlphabet{\mathsfit}{\encodingdefault}{\sfdefault}{m}{sl}
\SetMathAlphabet{\mathsfit}{bold}{\encodingdefault}{\sfdefault}{bx}{n}

\def\gA{{\mathcal{A}}}

\def\gE{{\mathcal{E}}}
\def\gF{{\mathcal{F}}}
\def\gG{{\mathcal{G}}}

\def\gL{{\mathcal{L}}}

\def\gN{{\mathcal{N}}}

\def\gR{{\mathcal{R}}}

\def\prob{{\mathbb{P}}}

\def\R{{\mathbb{R}}}

\begin{document}

\maketitle

\begin{abstract}
  Deep learning systems are known to exhibit \emph{implicit regularization} (alt. \emph{implicit bias}), favoring simple solutions instead of merely  minimizing the loss function.
In some cases, we can analytically derive the implicit regularization -- connecting it to an equivalent penalty that augments the learning objective.
However, modern deep learning systems are complex, carrying modifications to the training procedure and architecture (e.g. early stopping, minibatching, dropout) whose effects are not always directly interpretable.
Although estimating the resulting implicit regularization could aid theorists in algorithm design and practitioners in interpreting their hyperparameter choices, this problem has received little direct attention.
It is also tractable: regularization makes weight updates deviate from loss gradients, promising a signal for identifying implicit bias.
Here we provide gradient matching methods that can be used to empirically estimate the implicit regularization.
Our method works on networks with known regularization, recovering popular explicit penalties like $\ell_1$ and $\ell_2$.
It also replicates known implicit effects, like the quadratic weight penalty induced by early stopping in gradient descent, demonstrating that it can be used to test theories of implicit regularization.
Crucially, because our method is empirical, it can handle implicit regularization in arbitrary networks. We demonstrate this use by characterizing implicit $\ell_2$ effects of dropout in deep networks.
Our work shows that practitioners can use gradient matching to understand regularization in networks with implicit biases that are too complicated to derive analytically.

\end{abstract}

\vspace{-2ex}
\section{Introduction}
\vspace{-1ex}
Traditional statistical learning theory teaches us that adding more parameters to our predictive models increases the chances of overfitting to training data at the expense of their ability to generalize to novel examples.
One of the great challenges of modern machine learning is to understand how massively over-parametrized neural networks which can effectively memorize their training sets still manage to generalize \citep{zhangUnderstandingDeepLearning2021, powerGrokkingGeneralizationOverfitting2022}.
Recent progress has shown that the solution to the overfitting problem is for learning systems to favor simplicity: successful algorithms should prioritize parsimonious or highly-compressible solutions \citep{wilsonDeepLearningNot2025, rahamanSpectralBiasNeural2019, nakkiranSGDNeuralNetworks2019, valle-perezDeepLearningGeneralizes2019, yarasLawParsimonyGradient2023, maPrinciplesParsimonySelfConsistency2022}.
Neural networks therefore require constraints on their solution space -- whether explicit or implicit -- in order to achieve the impressive practical performance we observe.

We often \emph{explicitly} regularize or constrain models, adding \emph{inductive biases} which can improve generalization. Examples of explicit inductive biases include Bayesian priors \cite{ProbableNetworksPlausible1995, nealBayesianLearningNeural1996, wilsonBayesianDeepLearning2022}, weight decay \cite{hansonComparingBiasesMinimal1988}, and the translation invariance introduced in convolutional neural networks \cite{lecunGradientbasedLearningApplied1998}. Explicit regularization is one way of enforcing simplicity.

We also generally  \emph{implicitly} regularize.  A growing literature studies the implicit regularization effects of gradient-based learning \cite{soudryImplicitBiasGradient2018, aroraImplicitRegularizationDeep2019, aroraOptimizationDeepNetworks2018, chizatImplicitBiasGradient2020, neyshaburSearchRealInductive2015, nacsonImplicitBiasStep2022} and various algorithmic or architectural choices like early-stopping \cite{santosEquivalenceRegularizationTruncated1996, yaoEarlyStoppingGradient2007, ali2019continuous} or dropout \citep{wagerDropoutTrainingAdaptive2013, mianjyImplicitBiasDropout2018, mianjyDropoutNuclearNorm2019, cavazzaDropoutLowRankRegularizer2018, helmboldInductiveBiasDropout2015, weiImplicitExplicitRegularization2020}. Most theory casts implicit regularization in the form of a bias towards simpler functions (often via norm or rank minimization). It is believed that implicit regularization is essential to the success of deep learning in the over-parameterized regime \citep{bartlettDeepLearningStatistical2021a, wilsonDeepLearningNot2025, canatarSpectralBiasTaskmodel2021, woodworthKernelRichRegimes2020, alexanderNeuralNetworksNeed2025}.

Rigorous analysis of implicit regularization is generally limited to simple settings.
Much existing work focuses on tasks (deep matrix factorization \cite{gunasekarImplicitRegularizationMatrix2017, aroraImplicitRegularizationDeep2019}), models (shallow and/or linear networks \cite{mianjyImplicitBiasDropout2018, mianjyDropoutNuclearNorm2019, cavazzaDropoutLowRankRegularizer2018, chizatImplicitBiasGradient2020}), and algorithms (gradient flow or full-batch gradient descent \cite{chizatImplicitBiasGradient2020, minConvergenceImplicitBias2022, vardiMarginMaximizationLinear2022}) that differ significantly from those employed by deep learning practitioners.
The claim of these theoretical studies is often that the conclusions drawn in simplified settings might extend more generally, raising the need for empirical validation in larger networks.

In this paper we develop tools to empirically estimate implicit bias, with implications both for the theoretical study of learning algorithms and for training state of the art deep learning systems. Estimating implicit bias should allow theoreticians to empirically validate their analytical results and extend the study of implicit bias to complex networks that evade mathematical analysis, thereby providing a stronger bridge between theory and practice. For empirical researchers training deep neural networks, even heuristic estimates of implicit regularization can be a useful diagnostic.
While there is some literature on quantifying inductive bias, broadly construed, to our knowledge we are the first to directly estimate parametrized models of implicit bias as explicit regularizers.
We summarize and contrast other approaches with our own in Appendix \ref{app:related-work}. Explicitly modeling the bias permits directly testing theoretical models, and generally produces structured and interpretable results.

Our task is to find a mathematically expressible regularizer $\mathcal{R}$, such that explicitly optimizing the empirical loss $\mathcal{L}$ augmented by $\mathcal{R}$ is functionally the same as an implicitly biased procedure optimizing $\mathcal{L}$ alone. Given an a priori parametric family of regularizers, $\mathcal{R}(\cdot, \Lambda)$, we seek to estimate the values of the regularization parameters $\Lambda$. As a running example, $\ell_2$ regularization has been connected to common techniques like early stopping and dropout. This regularizer can be written as $\mathcal{R}(\theta, \lambda) = \lambda \| \theta \|_2^2$ where $\theta$ are the weights of a predictive model. Estimating $\lambda$ determines the strength of the technique's effective $\ell_2$ regularization.

Information about the implicit regularization is clearly embedded in the dynamics of learning.
Specifically, if the gradient of the training loss does not converge to zero, this may reflect additional constraints imposed by the optimization procedure which we seek to recover.
Alternatively, the same information can be gleaned at each timestep from discrepancies between the gradient of the loss and the update direction.
These deviations from simply going down the loss gradient are indicative of the underlying regularization.

We perform estimation by considering the learned weights to be a minimizer of a hypothetical objective combining the empirical loss \emph{and} the implicit regularizer. Alternatively, one can consider deviations of weight updates from the strict empirical loss gradient to be guided by the implicit regularizer.
Each of these perspectives implies a system of equations involving computable gradients of the loss and the candidate regularizer, in which the only unknowns are the regularization parameters $\Lambda$.
This framework is general, encompassing arbitrary features (weights, activations, data, etc.) and flexible parameterizations (any differentiable function class).

Our contributions are as follows:
\begin{itemize}
    \item \textbf{How to estimate implicit bias.} A novel empirical framework for estimating parametrized models of implicit regularization. We do so by matching the gradients of the loss to the gradients of a proposed regularizer, based on a decomposition of the empirical loss.
    \item \textbf{Characterizing successful estimation.} We distinguish between two modes of success: 1) recovering ground truth regularization and 2) recovering a regularizer that mimics the effect of the implicit regularization. We show that indeed we achieve (1) in simple settings where results are predicted by theory. We can evaluate (2) by estimating a regularizer and using it to retrain the model, then comparing the retrained model to the original. We then provide a theoretical characterization of the conditions needed to reproduce the original model.
    \item \textbf{From theory to practice.} Our proposed method bridges the gap between theory and practice. It establishes a method for theoreticians to empirically validate their theoretical results, demonstrated by our recovery of the implicit regularization of gradient descent in OLS.
    \item \textbf{An interpretability tool for training complex networks.} Our method also extends the study of implicit bias to complex networks that evade mathematical analysis. We demonstrate this by estimating the implicit regularization induced by dropout in deep nonlinear networks, showing how regularization varies with dropout rate and architecture.
\end{itemize}

\vspace{-1ex}
\section{Estimating implicit regularization}
\vspace{-1ex}
\subsection{Framework for estimation}
\paragraph{Setup.}
We study standard supervised learning, in which we receive input data $X \in \X$, and target $Y \in \Y$.
We consider a hypothesis space  $\mathcal{H}$
such that every $h \in \mathcal{H}$ is a function $h: \X \mapsto \Y$.
The hypotheses have some parametrization $h = h_\theta$ via the weights $\theta \in \Theta$ of the model we fit. As such the hypothesis space is defined by the chosen model's architecture. We train on a finite dataset $D = \left\{ (X_i, Y_i), i=1,\ldots,n\right\} \subseteq \X \times \Y$.
We implement some learning algorithm
$\mathcal{A} : \bigcup_{h=1}^\infty (\X \times \Y)^n \mapsto \mathcal{H}$
that returns a hypothesis $\hat h(D) \in \mathcal{H}$ with corresponding $\hat \theta(D) \in \Theta$.

\begin{wrapfigure}{r}{0.5\textwidth}
    \vspace{-1em}
    \centering
    \includegraphics[width=0.5\textwidth]{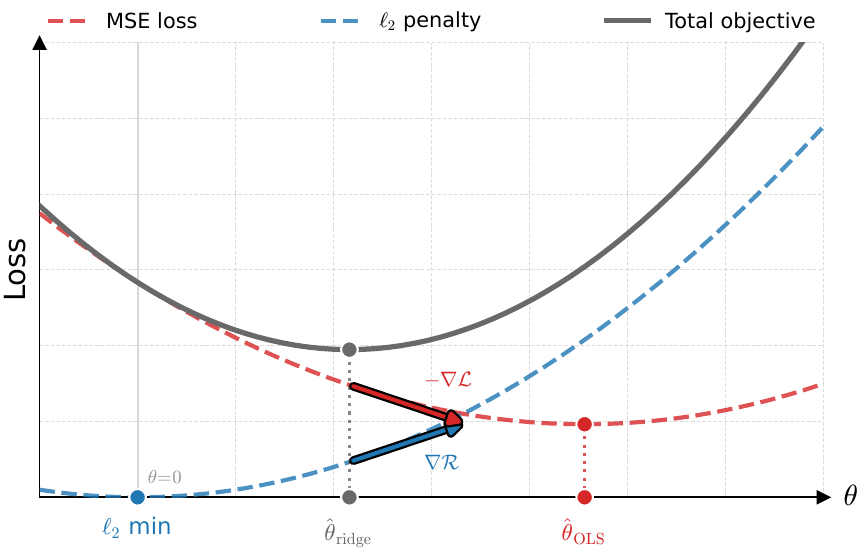}
    \caption{\textbf{Deviations from loss minima contain information about regularization.} Regularization modifies the objective function, leading to different solutions than strictly minimizing empirical loss. This gap must be explained by a regularizer ($\nabla \mathcal{R}$) counteracting any nonzero loss gradient ($-\nabla \mathcal{L}$) -- this is the key logic of the paper.}
    \label{fig:tradeoff-vis}
    \vspace{-1.8em}
\end{wrapfigure}

\vspace{-1ex}
\paragraph{A hypothetical minimum.} The standard empirical risk minimization (ERM) framework considers algorithms for minimizing the empirical risk (also called the loss or objective function) on the training data. However, while the ERM solution finds the model weights that satisfy $\argmin_{\theta \in \Theta} ~\mathcal{L}(\theta, D)$, we often obtain some $\gA(D) = \hat\theta(D)$ that deviates from the strict loss minimum. For example, we often stop training before the loss has fully converged.

Our key assumption is to view the residual loss as a quantity associated with an additional objective, namely the regularization objective. We understand implicit effects by assuming the learned $\smash{\hat{\theta}}$ is an exact solution to a \emph{regularized} empirical risk minimization (RERM) problem:
$$\hat{\theta}(D) = \argmin_{\theta \in \Theta} {\color{accent-red}\underbrace{\mathcal{L}(\theta, D)}_{\text{empirical loss}}} + {\color{accent-blue} \underbrace{\mathcal{R}(\theta,D)}_{\text{regularization term}}},$$
for some regularizer $\gR$. This is a \emph{hypothetical} condition, in that the network is explicitly trained to minimize the empirical loss. We simply hypothesize that due to some architectural or algorithmic choice, at a chosen point during training the network's weights can also be characterized as a solution to the above problem. 

\vspace{-1ex}
\paragraph{Computing and matching gradients.} If the loss and regularizer are differentiable, then $\hat\theta(D)$ must satisfy
\begin{equation}\label{eq:stationarity}
    \nabla_\theta \mathcal{L}({\theta},D)|_{\hat\theta(D)}
    +
    \nabla_\theta \mathcal{R}({\theta},D)|_{\hat\theta(D)}
    =0.
\end{equation}

Equivalently, we can see this as a \emph{gradient-matching condition} by rearranging:
\begin{equation}\label{eq:grad-matching}
    -\nabla_\theta \mathcal{L}({\theta},D)|_{\hat\theta(D)}
    = \nabla_\theta \mathcal{R}({\theta},D)|_{\hat\theta(D)}.
\end{equation}

The above states that the gradient of the regularizer must be approximately equal to the negative gradient of the loss. Intuitively, if we are at a minimum of an RERM objective, the net gradient in parameter space must be zero. If it were otherwise, we would move in the direction of that net gradient towards a better solution. So, if we assume the observed parameters $\smash{\hat{\theta}}$ are a minimizer, then the regularizer gradients \emph{must} negate the loss gradient!

Note that our stationarity condition implies a system of $p$ equations where $p$ is the dimensionality of $\theta$.
This implies that the approach in the following section can only fit regularizers with up to $p$ parameters -- the regularizer cannot have more parameters than the predictive model.
We discuss this and other challenges, along with suggested remedies, in Section \ref{sec:identifiability}. Appendix \ref{app:computing-grads} discusses computing the gradients in more detail: the loss gradient decomposes into a vector-Jacobian product that can be computed efficiently without materializing the full Jacobian.

\vspace{-1ex}
\paragraph{Learning the regularizer.}
We can learn the regularizer by parametrizing as $\mathcal{R}(\theta,D,\Lambda)$, and then solving
for $\Lambda$ in Equation
\ref{eq:grad-matching}.
In particular, we fit a \textbf{regularizer model} $\mathcal{R}({\theta}, D,\Lambda)$ where the predictive model's weights ${\theta}$ are fixed and the hyperparameters $\Lambda$ are learnable. \
Since our goal is that the above expression is small, we minimize the mean-squared-error of this gradient matching task:
\begin{equation} \label{eq:ber}
    \hat{\Lambda} = \argmin_\Lambda \{ \frac{1}{p} \|
    \nabla_\theta \mathcal{R}({\theta},D,\Lambda)|_{\hat\theta(D)}
    +
    \nabla_\theta \mathcal{L}({\theta},D)|_{\hat\theta(D)}  \|_2^2\}.
\end{equation}
This problem is general and allows us to use modern optimization tools like backpropagation to estimate arbitrary parameterization of $\Lambda$, but in many cases it reduces to a simple problem like linear regression and may permit closed-form solutions or more efficient implementations. 

\begin{wrapfigure}{r}{0.4\textwidth}
    \vspace{-2em}
    \centering
    \includegraphics[width=0.4\textwidth]{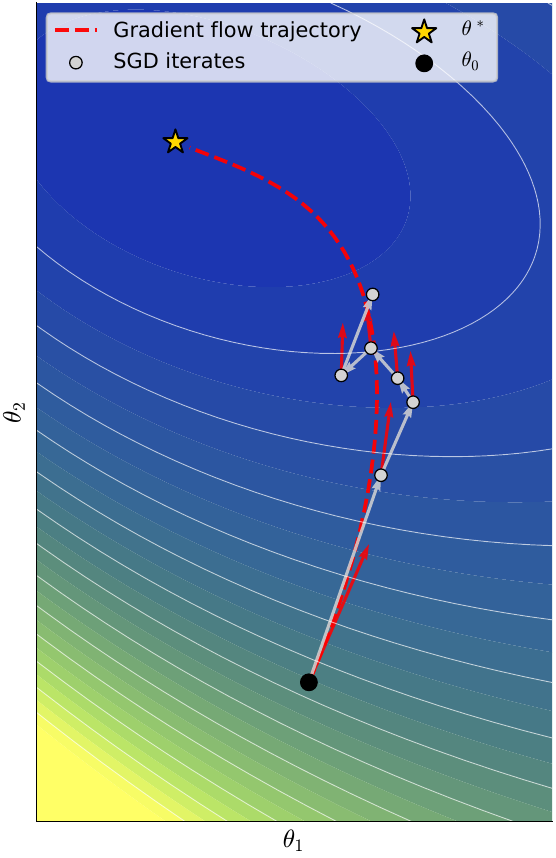}
    \caption{
    \textbf{Gradient deviations along the trajectory.} 
    In a 2-parameter toy model, SGD mini-batching causes weight updates along the trajectory to deviate from the full-batch gradient. We can model the implicit regularization underlying these deviations.}
    \label{fig:gradient-residuals}
    \vspace{-6ex}
\end{wrapfigure}

\vspace{-1ex}
\paragraph{Example: weight penalty in linear models.}
The simplest case we can consider is that of recovering a weight norm penalty in linear models, i.e. $\ell_1$ (lasso) or $\ell_2$ (ridge) regression. This case eases our analysis because we have closed-form solutions for the optimal weights.

Given input data $X \in \R^{N \times p}$ and target vector $y \in \R^N$, we want to learn $\theta \in \R^p$ such that we generate predictions
${\hat{y}}({\theta}) = X {\theta}$ which closely approximate the target vector.

We fit ridge regression with some $\lambda$:
\begin{align*}
    \nabla_\theta \left\{ \frac{1}{2} \|{y} - {\hat{y}}({\theta})\|^2 + \frac{\lambda}{2}  \|\theta \|^2 \right\} & = 0,
\end{align*}
resulting in:
\(~\hat \theta_\lambda = (X^\top X+\lambda I)^{-1} X^\top y\).

Now we consider the task of recovering $\lambda$. Formally, given data $X, y$, and weights $\hat \theta_\lambda$ we want to find the solution $\hat{\lambda}$ that satisfies:
\(
-X^\top (y - X\hat{\theta}_\lambda) + \hat{\lambda} \hat \theta_\lambda = 0_p
\).

This gives us a $p$-dimensional system of linear equations for $\hat{\lambda}$, which holds exactly when $\hat{\lambda} = \lambda$.

\vspace{-1ex}
\paragraph{Deviations along the trajectory.}
An alternative perspective examines implicit regularization at each optimizer step (as in \cite{barrettImplicitGradientRegularization2022, smithOriginImplicitRegularization2021}). At a given point during the learning process, weights do not strictly follow the loss gradient; instead they follow a modified trajectory induced by the many components of the optimization process.
\begin{equation}
\label{eq:trajectory-gradient-matching}
\Delta\theta = \theta_{t+1} - \theta_{t} = -\eta \nabla_\theta (\mathcal{L} + \mathcal{R}).
\end{equation}
Rearranging this expression implies that regressing $-\eta^{-1} \Delta\theta - \nabla_\theta \mathcal{L}$ on $\nabla_\theta \mathcal{R}$, where each point in the trajectory yields one datapoint, offers an alternative method for learning a regularizer. We further describe this procedure and provide theoretical guarantees for the above in the form of a martingale generalization bound within Appendix \ref{app:single-trajectory}. See Section \ref{sec:igr} for an application to estimating the implicit gradient regularization proposed by \cite{barrettImplicitGradientRegularization2022}.

Figure \ref{fig:gradient-residuals} illustrates the trajectory deviations (in a 2-parameter model) induced by SGD mini-batching. Note that this framing contains training endpoints as a special case ($\Delta \theta_t \approx 0$), so forcing $\Delta \theta = 0$ by assumption recovers the ``hypothetical minimum'' perspective we described above.

\vspace{-1ex}
\paragraph{Evaluating the estimated regularizer}
\label{sec:eval}
Once we fit $\hat{\Lambda}$, how can we decide if it is ``correct'' or even a good approximation?
We posit that there are two main modes of success:
\begin{enumerate}
    \item \textbf{Recovery of hypothesized $\Lambda$.} In settings where we know the ground truth regularization or have strong theoretical predictions about its exact parametrization, we should recover estimates $\smash{\hat{\Lambda}}$ from \eqref{eq:ber} that exactly match theory.
    \item \textbf{Recovery of original model weights $\hat{\theta}$: validation via explicit regularization.} Say we are interested in the implicit bias of an algorithmic or architectural feature that can be added to, or removed from, the training procedure. If we do not have any ``ground truth'' $\Lambda$ to work with, another way of validating our empirical estimates $\hat{\Lambda}$ is to remove (or ablate) the training feature of interest and retrain the original predictive model $f$ with an explicitly regularized objective, namely $\mathcal{L}(\theta) + \mathcal{R}(\theta, \hat{\Lambda})$. If $\mathcal{R}(\theta, \Lambda)$ accurately captures the effective regularization, then this should recover the same solution for the optimal weights $\hat{\theta}$.
\end{enumerate}

That is, the weights learned with implicit regularization are equivalent to the weights learned when early stopping is replaced with the estimated regularizer. While this doesn't necessarily mean that the estimated regularizer was used during training, it does mean that the two approaches are functionally equivalent  on the training set. The second approach has been employed in prior literature, e.g. \cite{weiImplicitExplicitRegularization2020}. Appendix A provides a detailed discussion of theoretical guarantees for recovering the original model weights. Intuitively, if the loss converges to a relatively flat minimum, local convexity guarantees that the solution is unique and we will recover the model.

For assessing the fit itself, we can use a diagnostic analogous to regression $R^2$, measuring the proportion of the squared loss-gradient magnitude explained by the estimated regularizer:
\begin{equation}
\label{eq:gradient-fit}
    R_{\mathrm{grad}}^2 = 1 - \frac{\|\nabla_\theta \mathcal{L}(\hat\theta,D) + \nabla_\theta \mathcal{R}(\hat\theta,D,\hat\Lambda)\|_2^2}{\|\nabla_\theta \mathcal{L}(\hat\theta,D)\|_2^2}.
\end{equation}
This is an uncentered $R^2$: the baseline is a zero gradient rather than the mean across gradient components. It equals one for exact gradient matching and zero when the fitted regularizer offers no improvement over a zero regularizer. Poor fit can reflect an inadequate candidate family, noisy gradients, or optimization error, while high fit alone does not establish recovery of the original model; see Appendix~\ref{app:gradient-fit} for practical guidance.

\subsection{Regularizer identification}
\label{sec:identifiability}

\vspace{-1ex}
\paragraph{Which regularizers can be modeled?}
In order to characterize the implicit bias of some additive module $\mathcal{M}$ like an architecture (dropout layers) or algorithm (gradient descent/early stopping), you need some notion of training \emph{without} $\mathcal{M}$ -- we might say that it is \emph{ablatable} or amenable to ablation study, so that it may be replaced by an equivalent explicitly regularized problem. Appendix~\ref{app:scaling} discusses attribution to individual components and their interactions, and extension to larger systems.

\vspace{-1ex}
\paragraph{Collinearity of gradients.}
A natural application of our method is to simultaneously fit a number of candidate regularizers with scale coefficients -- essentially turning the problem into linear regression over a set of regularizer gradients. That is:
\[
    -\nabla_{\theta} \mathcal{L} = \vec{\lambda}^\top \nabla_{\theta} \vec{\mathcal{R}} = \lambda_1 \nabla_{\theta} \mathcal{R}_1 +  \ldots + \lambda_r \nabla_{\theta} \mathcal{R}_r.
\]

The problem of accurately estimating the coefficient vector $\vec{\lambda}$ is more challenging when the gradients (i.e. the regression features) exhibit collinearity. 
In particular, if two regularizer gradients are perfectly collinear then the covariance matrix $(\nabla_{\theta} \mathcal{R}^T\nabla_{\theta} \mathcal{R})$ is singular and the OLS estimator is undefined. When the regularizer gradients are nearly collinear, the estimator is well-defined and unbiased, but has inflated variance, making estimation unstable. We elaborate and provide theoretical guarantees for both this parametric and the associated nonparametric case in Appendix \ref{app:pac}. There we show that estimation error scales in $O(\sqrt{r/p_{\operatorname{eff}}})$, where in the linear parametric case $r$ is the number of regularizers and  $p_{\operatorname{eff}}$ is the effective parameter count (this is analogous to effective sample size, since parameters are samples in our regression of $\nabla_{\theta} \mathcal{L}$ onto $\nabla_{\theta} \mathcal{R}$).

\vspace{-1ex}
\paragraph{Improving estimation with more diverse sampling.}
The approach described thus far only allows estimation of regularizers with dimension at most that of the predictive model,  \( p = \dim(\theta) \). A natural goal for improving estimation and enabling more complex parameterizations is to augment with additional estimating equations. Intuitively, adding more equations helps if they introduce additional regularizer gradients that do not overlap with existing ones (analogous to adding more i.i.d. data samples in regression). As mentioned above, estimation error scales in $O(\sqrt{r/p_{\operatorname{eff}}})$ -- adding more equations increases $p_{\operatorname{eff}}$, and correlation among gradients decreases $p_{\operatorname{eff}}$.

In an ideal setting, we get more equations by retraining the same model on a new dataset. The new training endpoint will have new weights $\hat{\theta}$ and therefore a new set of features $~\nabla_\theta \mathcal{R}|_{\hat\theta(D)}~$ and targets $~-\nabla_\theta \mathcal{L}({\theta},D)|_{\hat\theta(D)}~$ to concatenate to the regression. In synthetic data settings, we can sample new datasets that are independent, maximizing the amount of additional information obtained from each new sample. We will later demonstrate this works for estimating the regularization due to early stopping in Figure \ref{fig:ols-early-stopping}B. In realistic data settings, many datasets (especially those from the same domain, like images) are correlated and therefore will lead to more similar $\hat{\theta}$ -- but even sampling variability in $\smash{\hat{\theta}}$ will improve our coverage.

Another natural way to obtain additional estimating equations, when constrained to a single training dataset, is through bootstrap resampling: each bootstrap replicate yields a new fitted parameter vector and thus a new approximate stationarity condition, expanding the gradient equations available for estimating \(\Lambda\). 
Moreover, one need not resample uniformly -- it is also possible to weight data samples in a way that stabilizes estimation, like upweighting samples for which the gradients are less collinear. See Appendix \ref{app:bootstrapping} for more detail and the bootstrapping algorithm for this case, as well as an example of bootstrapping applied to early stopping in OLS, as studied in Section \ref{sec:gd-ols}.

\vspace{-1ex}
\paragraph{Estimating over trajectories.}
When training models, we can compute the loss and regularizer gradients at any point during training.
While Equation \ref{eq:stationarity} will hold most closely as the model training converges, there is nothing to stop us from running our procedure earlier on in training.
This opens two possible targets for estimation:
\begin{enumerate}
    \item \textbf{Constant $\Lambda$}. Assuming that the regularizer function does not change over time, the trajectory can be used to augment the estimation of Equation \ref{eq:trajectory-gradient-matching} by contributing more gradient-matching terms -- much like bootstrap replicates. See Appendix \ref{app:single-trajectory} for a detailed analysis.
    \item \textbf{Time-varying $\Lambda^{(t)}$}. It is also possible that the implicit regularization changes as a function of time (or optimization steps). Then, each point along the trajectory represents a separate estimation problem. This allows for tracking how the implicit regularization evolves over the course of training, which is much more general than simply studying a model at convergence. This framing is potentially more interesting to researchers studying training dynamics rather than generalization. One can employ either of Equations \ref{eq:grad-matching} or \ref{eq:trajectory-gradient-matching} at each timepoint, but the hypothetical stationarity assumed for \ref{eq:grad-matching} is probably a poor approximation early on in training. Section \ref{sec:gd-ols} studies a time-varying quadratic weight penalty induced by gradient descent for OLS.
\end{enumerate}
Appendix~\ref{app:training-duration} examines sensitivity to training duration in the dropout experiment.

\vspace{-1ex}
\section{Empirical Studies}
\vspace{-1ex}
The following experiments test recovery of known explicit regularization (elastic net) and theorized implicit regularization (from early stopping and discrete GD steps), and assess the adequacy of a scalar ridge candidate for dropout-trained networks.

\begin{figure}[h!]
    \centering
    \includegraphics[width=\linewidth]{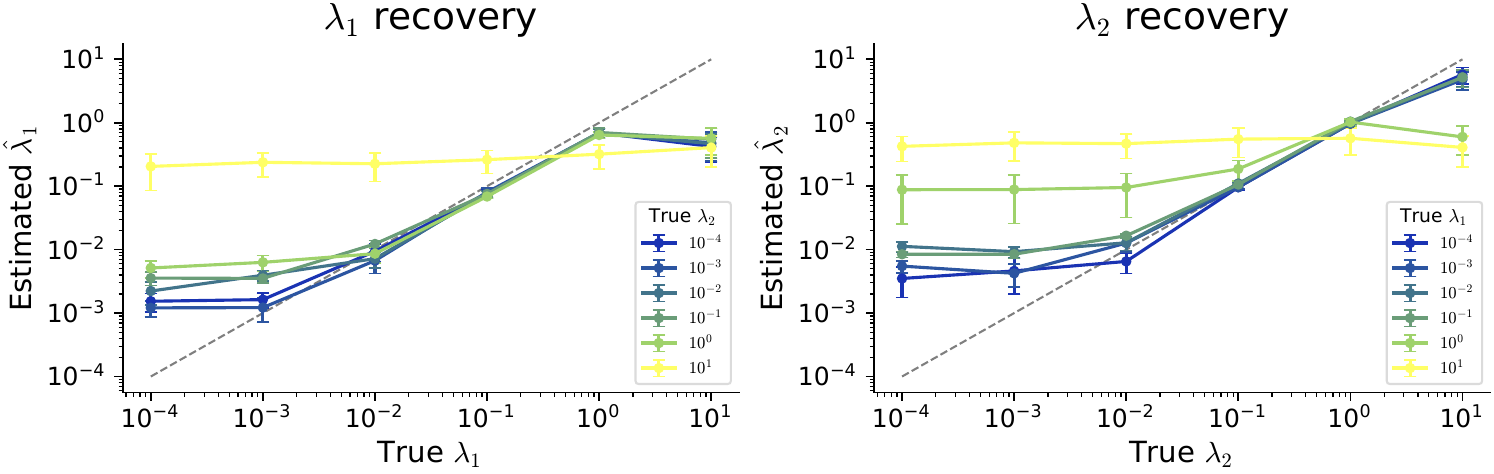}
    \caption{\textbf{Recovering explicit regularization.} Elastic-net recovery at $\beta=10^{-3}$ on a $6\times 6$ grid of true $(\lambda_1,\lambda_2)$ with $10$ independent dataset draws per cell. Each panel plots the estimated $\hat\lambda_i$ against the true $\lambda_i$ on log--log axes; lines show the mean across seeds and error bars give $\pm 1$ standard error. \textbf{Left:} $\hat\lambda_1$ vs.\ true $\lambda_1$, one line per value of true $\lambda_2$. \textbf{Right:} $\hat\lambda_2$ vs.\ true $\lambda_2$, one line per value of true $\lambda_1$. The dashed diagonal marks perfect recovery.}
    \label{fig:elasticnet-recovery}
    \vspace{-1ex}
\end{figure}

\vspace{-1.5ex}
\paragraph{Explicit elastic net regularization.}
\label{sec:elastic-net}

In order to sanity-check our approach, we first tasked our framework with recovering the parameters of a neural network explicitly trained with regularization.
We trained deep ReLU networks with various combinations of $\ell_1$ and $\ell_2$ regularization (also known as elastic net).

Figure \ref{fig:elasticnet-recovery} shows the recovered penalties as a function of the true penalties.
Note this is a difficult task because smoothing $\ell_1$ near zero makes the two norms indistinguishable for very small weights ($\leq 10^{-3}$). Consequently, large penalties (which drive many weights towards zero) reduce identifiability. For very small penalties, the regularization has little effect on the optimization procedure and there is not much signal for estimation. For intermediate values of the penalty parameters, however, recovery is quite accurate.


\begin{figure}
    \centering
    \includegraphics[width=\linewidth]{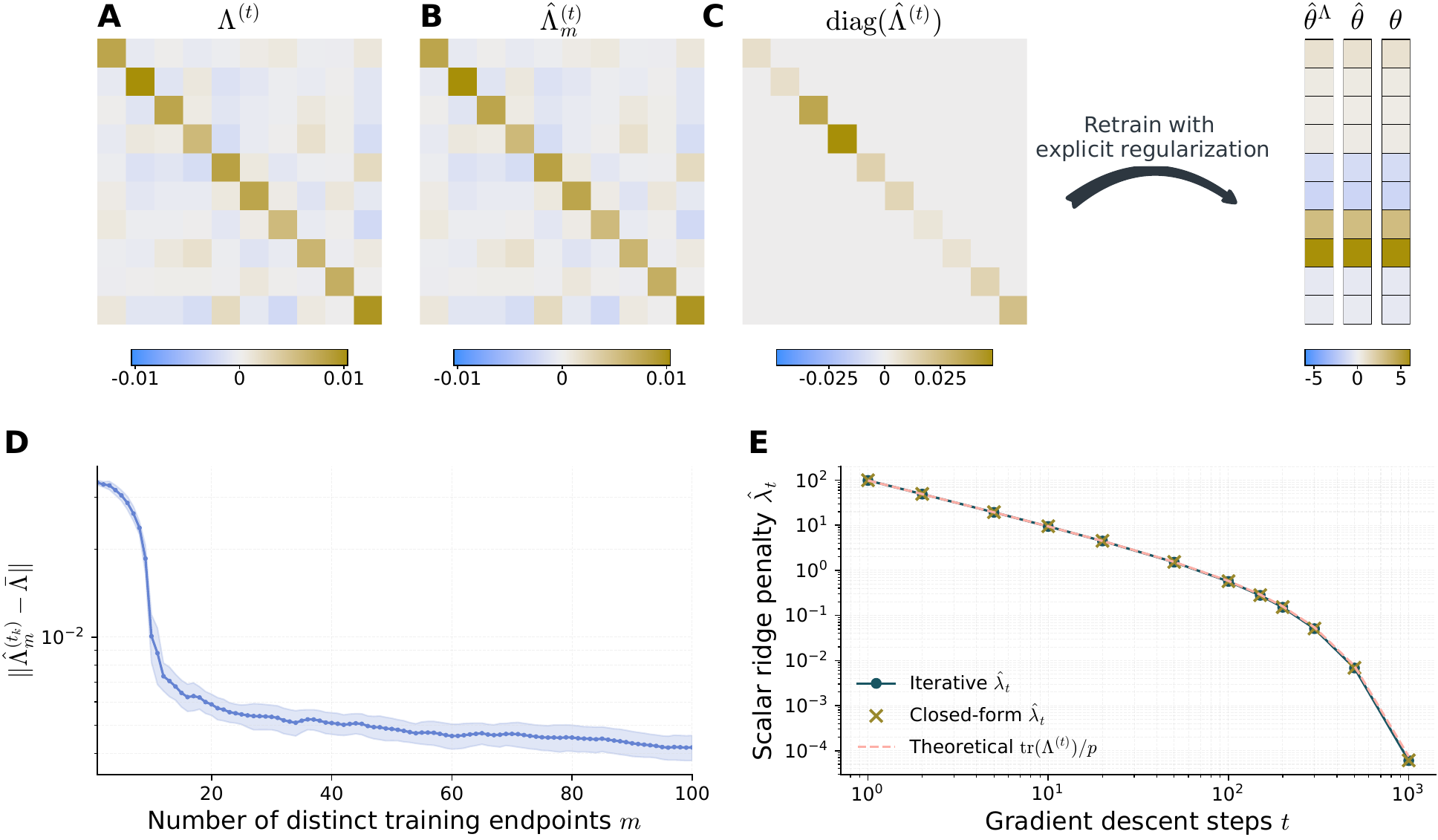}
    \caption{
        \textbf{Implicit regularization due to early stopping.}
        \textbf{(A)} The theoretical regularization matrix $\Lambda^{(t)}$ predicted by Ali et al.\ at the early-stopped iterate $t=500$.
        \textbf{(B)} The full symmetric estimator $\hat{\Lambda}_m^{(t)}$ using $m=10$ endpoints, each trained to the same fixed iterate $t=500$. 
        \textbf{(C)} The under-parametrized diagonal estimator $\mathrm{diag}(\hat{\Lambda}^{(t)})$ fit from a \emph{single} endpoint. The regularized solution $\hat{\theta}^\Lambda$ under $\mathrm{diag}(\hat{\Lambda}^{(t)})$ matches the original early-stopped trained weights $\hat\theta$, and the ground-truth coefficients $\theta$. 
        \textbf{(D)} Convergence of the symmetric-matrix estimator under per-endpoint early stopping: $\|\hat{\Lambda}_m^{(t_k)} - \bar{\Lambda}\|$ as the number of endpoints $m$ grows from $1$ to $100$, where $\bar{\Lambda}$ is the median of the endpoint-specific matrices $\Lambda^{(t_k)}$. Each endpoint is stopped individually when its training loss plateaus, so stop steps $t_k$ vary. Band shows 95\% CI. 
        \textbf{(E)} The single-parameter heuristic estimate $\hat\lambda_t$ tracked over gradient descent steps $t$.
    }
    \label{fig:ols-early-stopping}
    \vspace{-1em}
\end{figure}

\vspace{-1.5ex}
\paragraph{Early stopping in OLS.}
\label{sec:gd-ols}
It is nice to recover known explicit regularization, but the real goal is to estimate \textit{implicit} regularization that is not in our objective function during training. To that end we study a simple example with well-understood theory: the effect of early stopping in full-batch gradient descent for ordinary least-squares linear regression \citep{santosEquivalenceRegularizationTruncated1996, yaoEarlyStoppingGradient2007, ali2019continuous}. We analyze synthetic data generated from a noisy linear model $y = X \theta + \varepsilon$, where $X \in \R^{1000 \times 10}$ has entries $X_{ij} \stackrel{\text{iid}}{\sim} \N(0,1)$, coefficients $\theta_i \stackrel{\text{iid}}{\sim} \N(0, 3^2)$, and noise $\varepsilon_i \stackrel{\text{iid}}{\sim} \N(0,1)$.

\citet{ali2019continuous} predicts that the implicit effect of early stopping is a form of data-dependent quadratic regularization that decays over training.
Consider GD with step size $\eta$ on the OLS objective, starting at $\hat\theta^{(0)} = 0$, and assume $\eta< 1/\|X^\top X\|_{\op}$. Let $\tfrac{1}{n} X^\top X = V S V^\top$ be the singular value decomposition.  
For each $t = 1,2,3,\dots$, the iterate $\hat{\theta}^{(t)}$ from step $t$ in gradient descent uniquely solves
\begin{equation}
\label{eq:early-stopping-theory}
    \min_{\theta \in \mathbb{R}^p} \; \frac{1}{n}\|y - X\theta\|_2^2 + \theta^\top \Lambda^{(t)} \theta, \quad\quad
    \Lambda^{(t)} = V S\big((I - \eta S)^{-t} - I\big)^{-1} V^\top.
\end{equation}

To recover this implicit regularization exactly, we need to learn the matrix $\Lambda^{(t)}$, which has $p(p+1)/2$ free parameters (it is symmetric and PSD).
We only have $p$ equations in the system of equations from a single training endpoint, posing a challenge for recovering the full matrix. 

We present two approaches in Figure \ref{fig:ols-early-stopping} and show how they lead to the two modes of success described in Section \ref{sec:eval}: recovery of hypothesized $\Lambda$ (\ref{fig:ols-early-stopping}B) and recovery of the original model weights $\hat{\theta}$ (\ref{fig:ols-early-stopping}C).  
Figure \ref{fig:ols-early-stopping}A shows the quadratic weight penalty theoretically induced by early stopping according to Eq. \ref{eq:early-stopping-theory}. In \ref{fig:ols-early-stopping}B, we show that concatenating training endpoints from identical models trained on $m$ distinct datasets (as discussed in Section \ref{sec:identifiability}) provides enough coverage to recover the full matrix exactly. Estimation error falls off steeply as a function of $m$ (see \ref{fig:ols-early-stopping}D). In \ref{fig:ols-early-stopping}C, we use just a single training endpoint and force $\hat{\Lambda}^{(t)}$ to be diagonal so that it has only $p$ parameters. We show that, though the underparametrized version diverges from the matrix predicted by theory, using it as an explicit penalty in the regression \emph{without} early stopping yields exactly the same weights ($\hat{\theta}^\Lambda$) as the original unregularized model \emph{with} early stopping ($\hat{\theta}$).

We next asked if an underparametrized model of implicit regularization could still be interpretable. After all, the true implicit regularization in a complex network might be nearly impossible to write in closed form. We may only be able to fit a heuristic model that captures some aspect or summary of its myriad effects. 
To test this, we fit a model with a scalar $\ell_2$ regularization parameter, i.e. $\Lambda_t = \lambda_t I$. Our theory of early-stopping in OLS tells us that $\lim_{t\to \infty} \alpha^\top \Lambda_t \alpha \to 0$ because the matrix $\left((I-\eta S)^{-t}-I\right)^{-1}$ vanishes for large $t$ (given that eigenvalues $s_i$ are positive and $\eta s_i < 1$). We should be able to capture this behavior with our heuristic model; that is, if the theory is correct we should see $\hat{\lambda}_t \to 0$.
Figure \ref{fig:ols-early-stopping}E shows that our heuristic estimate decays over time, with an extremely tight match to theory. 
We take this as an encouraging sign that we can use underparameterized models of implicit regularization to test our theory and intuitions.


\begin{figure}
    \centering
    \includegraphics[width=0.9\linewidth]{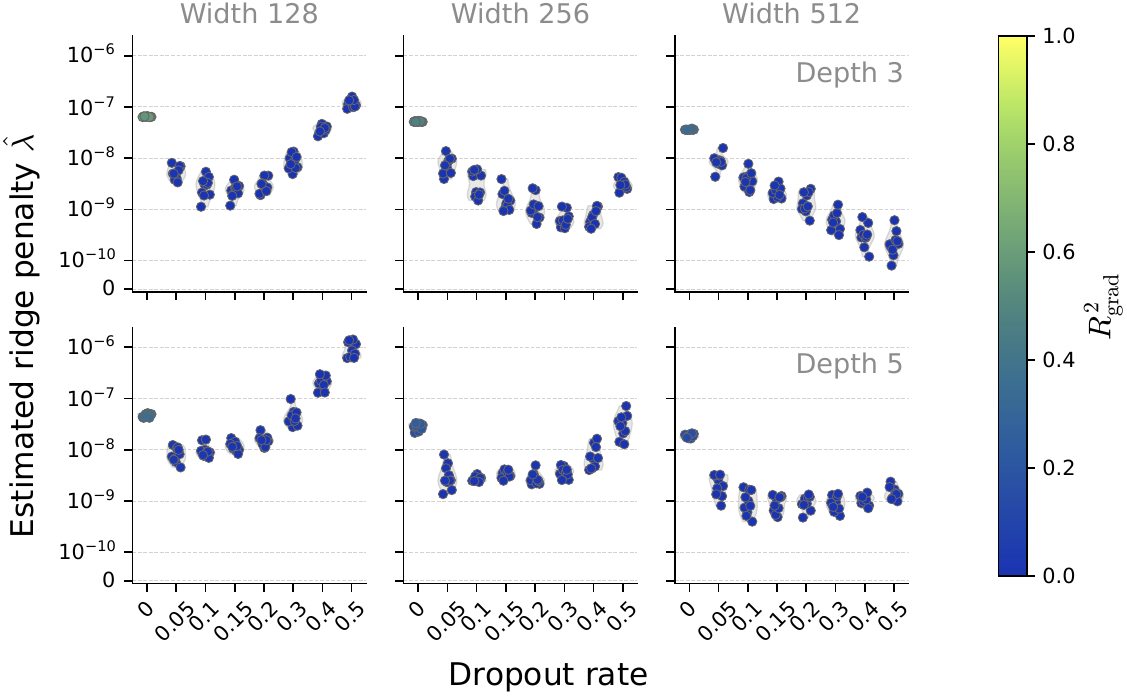}
    \caption{\textbf{Endpoint ridge projections for dropout-trained classifiers.} Exact nonnegative gradient-matching estimates $\hat\lambda$ for the candidate regularizer $\mathcal{R}(\theta,\lambda)=\lambda\|\theta\|_2^2$ on MNIST. Each point is one of ten seeds; columns vary width, rows vary depth, and gray violins summarize the seed distributions in displayed coordinates. Color indicates the fraction of squared clean-training loss-gradient norm explained, $R_{\mathrm{grad}}^2$ in \eqref{eq:gradient-fit}, on a shared zero-to-one scale (lighter = better fit). The vertical axis is logarithmic above $10^{-10}$ and linear near zero. Zero-dropout controls are retained.}
    \label{fig:dropout-l2}
    \vspace{-1.5em}
\end{figure}

\vspace{-1.5ex}
\paragraph{Dropout.}
\label{sec:dropout}
Dropout \citep{srivastavaDropoutSimpleWay2014} randomly zeroes out hidden units during training and is broadly understood to act as an adaptive regularizer that penalizes over-reliance on specific pathways. Several theoretical analyses have shown that, for linear models, dropout is akin to adaptive $\ell_2$ regularization whose strength depends on the dropout rate and the second moments of the activations \citep{wagerDropoutTrainingAdaptive2013, weiImplicitExplicitRegularization2020}. Here we trained a series of neural network classifiers on MNIST \cite{lecun2010mnist} with varying levels of dropout.

Figure~\ref{fig:dropout-l2} shows the estimated ridge coefficient $\hat\lambda$ as a function of dropout rate, faceted by depth and width. Among positive dropout rates, narrower networks exhibit a pronounced upturn in the fitted coefficient at high dropout. Wider networks generally have smaller coefficients, with flatter or decreasing trends depending on depth. The increase in narrower networks is consistent with the intuition that stronger dropout can exert greater regularization pressure, although this relationship is not uniform across architectures.

However, the low gradient-fit $R_{\mathrm{grad}}^2$ values indicate that the estimated projection onto the proposed ridge regularizer explains only a small proportion of the squared residual loss-gradient magnitude. This suggests that scalar ridge provides an incomplete description of dropout's complex, data-dependent effects. Indeed, replacing dropout with each model's fitted ridge penalty and retraining from the same initialization under the same optimization and stopping procedure did not recover the original dropout model's test performance (Appendix~\ref{app:dropout-replacement}). Though estimated regularizers can still provide insight into the training dynamics, they may not lead to retrained models that match the original behavior. Here, the projection reveals interpretable variation in the ridge-aligned component across training configurations, while the goodness-of-fit and retraining results show the limitations of using that component to fully explain dropout.

\vspace{-1.5ex}
\paragraph{Gradient step size.}
\label{sec:igr}
\citet{barrettImplicitGradientRegularization2022} describe how the step size ($\eta$) of discrete gradient descent implicitly regularizes models by scaling a penalty on trajectories with large loss gradients.
Specifically, they propose that gradient descent in $\R^p$ minimizing a loss $\mathcal{L}$ deviates from the exact continuous path of gradient flow. It is actually closer to the exact continuous path of a modified loss
$\mathcal{L}(\theta) + \mathcal{R}(\theta, \lambda)$
where $\lambda \equiv \frac{\eta \cdot p}{4}$
and \(\mathcal{R} (\theta, \lambda) \equiv \frac{\lambda}{p} \sum_{i=1}^p \left( \nabla_{\theta_i} \mathcal{L}(\theta)\right)^2\).

Figure \ref{fig:igr} uses this setting as a known-bias recovery check. Rather than treating $\lambda = \eta p/4$ as given, we estimate $\lambda$ from trajectory deviations and compare $\hat\lambda$ to the analytic value derived by \citet{barrettImplicitGradientRegularization2022}. In this experiment, Eq. \ref{eq:trajectory-gradient-matching} is instantiated with the continuous gradient-flow path of L as the unregularized reference, rather than the Euler step $-\eta\nabla \mathcal{L}$; Appendix \ref{app:igr-methods} details how we compute this reference so that the residual captures the finite-step IGR correction. We also reproduce the qualitative Figure 2 relationships between stronger implicit gradient regularization, smaller averaged squared gradients, and higher test accuracy. 

This application shows that our method can be used to estimate regularization in the \emph{dynamics} of learning along the training trajectory, including early on in training. This supplements our earlier applications estimating regularization in models near convergence. 

\begin{figure}
    \centering
    \includegraphics[width=\linewidth]{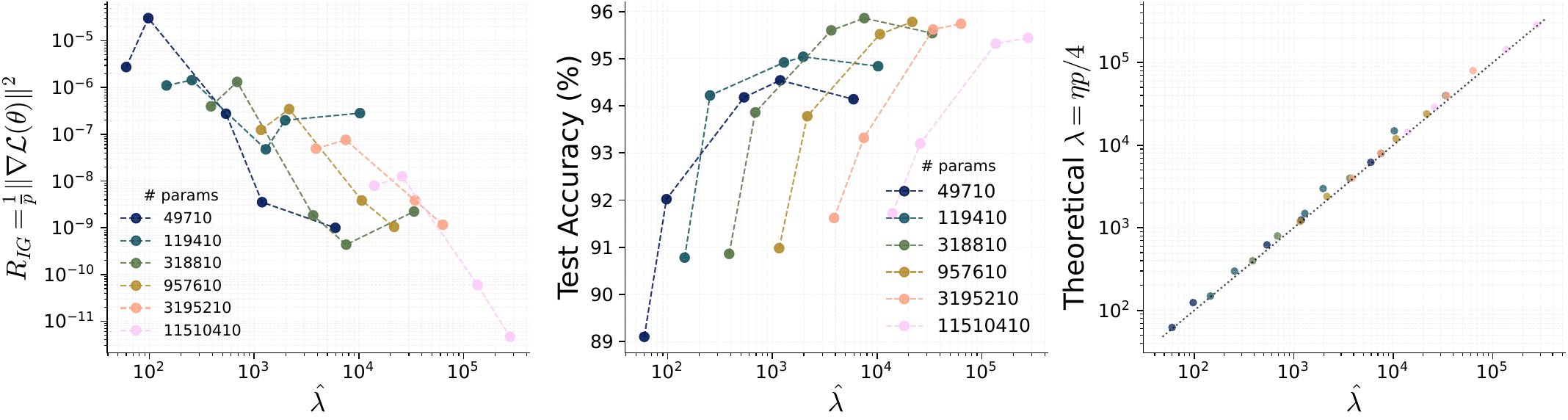}
    \caption{\textbf{Recovering implicit gradient regularization from discrete gradient steps.} \citet{barrettImplicitGradientRegularization2022} derived $\lambda=\eta p/4$ for the gradient-penalty regularizer induced by discrete GD. \textbf{Left:} Value of $R_{IG}$, the averaged squared gradient, against the estimated scaling $\hat{\lambda}$. \textbf{Middle:} Test accuracy against $\hat{\lambda}$. \textbf{Right:} Direct recovery check comparing $\hat{\lambda}$ to the theoretical value $\eta p/4$.}
    \label{fig:igr}
    \vspace{-1em}
\end{figure}



\vspace{-1ex}
\section{Conclusions}
\vspace{-1ex}
We have shown that it is possible to estimate the regularization implicitly embedded in a learning system. We accomplish this by assuming that a system jointly optimizes the empirical loss \emph{and} a regularizer. This assumption implies a relationship among the empirical loss gradients, regularizer gradients, and observed weight updates that permits solving for unknown regularization parameters. We demonstrate that this approach can recover explicit regularization as well as theoretically predicted implicit regularization.

Our approach, as it ultimately is an instance of parametric statistical modeling, has familiar limitations. In particular,  our capacity to fit models with many parameters is bound to our effective sample size. In this setup, sample size corresponds to the number of equations given by gradient matching (typically the number of parameters $p$ in the original predictive model). In Section \ref{sec:identifiability} we discussed ways to increase the sample size, but ultimately they may be limited by the data at hand. We also face familiar identifiability problems -- regularization parameters associated with collinear gradients yield estimates with high variance. It is difficult to distinguish reliably between classes of regularizers with similar gradients. In Appendix \ref{app:pac} we give PAC bounds for both the parametric and nonparametric cases. Additionally, we frame implicit regularization as a phenomenon which can be replaced by an equivalent explicit regularizer, when estimated correctly -- but in some cases it might not be feasible or sensible to replace the training technique being studied. However, many deep learning techniques can be entirely ablated, making them amenable to study.


This paradigm opens exciting avenues for future research into the implicit regularization effects of the many complex innovations and engineering tricks used in state of the art deep learning systems. Practitioners often develop training techniques in a mostly ad-hoc manner, greedily testing various training configurations and advocating for techniques that empirically improve generalization on available validation datasets. This clearly works in the long run, as the impressive evolution of deep learning proves. However, as model size and compute budgets scale, training runs become massively expensive. Estimating implicit regularization could provide a relatively cheap and practical way for practitioners to study the effects of various training configurations early on in training, by studying how the dynamics deviate from strict loss minimization. This is another axis along which to evaluate models, which would complement existing heuristics like validation loss, weight and gradient norms, or perplexity. We consider our work here to be just a starting point for empirically characterizing the implicit regularization in deep learning systems, which we believe to be essential to their capacity to generalize beyond their training data.



\newpage

\begin{ack}
This research was supported in part by a Susquehanna PhD Fellowship awarded to Joseph H.~Rudoler by Susquehanna International Group.
The authors report no other direct funding or third-party support for this work and no other competing financial interests during the 36 months preceding submission.
\end{ack}

\bibliographystyle{plainnat}
\bibliography{ref}

\newpage
\appendix

\section{Uniqueness and recovery with an estimated explicit regularizer}
\label{app:unique-weights}

In this appendix we characterize how retraining a model with an explicit regularizer $\mathcal{R}$ with estimated parameters $\hat{\Lambda}$ can often recover the original predictive model parameters $\hat{\theta}$ induced by implicit regularization. 
This was motivated by a surprising observation in OLS: fitting explicitly regularized least-squares with a diagonal version of the implicit regularization matrix reproduced the same weights as unregularized early stopping. So not only did an explicit regularizer reproduce the predicted implicit effects, but even an \emph{misspecified} (underparameterized) explicit regularizer reproduced the same effects!

Here we show that under various convexity and curvature assumptions, fitting $\hat{\Lambda}$ such that the resulting regularized loss is minimized at $\hat{\theta}$ actually produces a unique (or at least locally unique) solution. Therefore, retraining a model on that regularized loss recovers exactly $\hat{\theta}$. The most lenient of these assumptions -- that the regularized objective has positive definite Hessian at $\hat{\theta}$ -- is quite plausible in our estimation. A rich literature suggests that flatter minima have desirable properties which lead to improved generalization \cite{hochreiterFlatMinima1997, kaurMaximumHessianEigenvalue2022}, and consequently many modern training techniques (like choice of batch size \cite{keskarLargeBatchTrainingDeep2017}) optimize for flatter minima either intentionally \cite{foretSharpnessAwareMinimizationEfficiently2021, zhangGradientNormAware2023, chaudhariEntropySGDBiasingGradient2017} or by happy coincidence.  

Fix an estimated hyperparameter $\hat\Lambda$. Let \(\mathcal{L}:\mathbb R^p\to\mathbb R\) be \(C^2\), and suppose that for the fixed
\(\hat\Lambda\), the map \(\mathcal{R}: \R^p \to \R \) is \(C^2\).

Define
\[
    f(\theta) \coloneqq \mathcal{L}(\theta) + \mathcal{R}(\theta,\hat\Lambda),
    \qquad
    \mathcal{F}(\theta) \coloneqq \nabla f(\theta).
\]
Assume unconstrained optimization over $\Theta=\mathbb{R}^p$.

\paragraph{Local uniqueness via local convexity.}
\begin{proposition}[PD Hessian $\Rightarrow$ isolated stationary point]
    \label{prop:local}
    If $\mathcal{F}(\hat\theta)=0$ and $\nabla^2 f(\hat\theta)\succ0$, then:
    (i) $\hat\theta$ is a strict local minimizer of $f$;
    (ii) there exists $r>0$ such that $\hat\theta$ is the unique zero of $\mathcal{F}$ in $B_r(\hat\theta)$.
\end{proposition}


\begin{proof}
    Let \(H(\theta)=\nabla^2 f(\theta)\). Since \(H\) is continuous and
    \(H(\hat\theta)\succ0\), there exist \(\mu>0\) and \(r>0\) such that
    \(H(\theta)\succeq \mu I\) for all \(\theta\in B_r(\hat\theta)\).
    Hence \(f\) is \(\mu\)-strongly convex on the convex ball \(B_r(\hat\theta)\).
    Because \(\nabla f(\hat\theta)=0\), \(\hat\theta\) is the minimizer of \(f\) on
    \(B_r(\hat\theta)\), and strong convexity makes this minimizer unique.
    Therefore \(\hat\theta\) is a strict local minimizer and the unique zero of
    \(\nabla f\) in \(B_r(\hat\theta)\).
\end{proof}

\paragraph{Global uniqueness: two sufficient routes.}
\begin{proposition}[Strong convexity]
    \label{prop:global-strong}
    Assume $f$ is $\mu$-strongly convex on $\mathbb{R}^p$ (i.e., $\nabla^2 f(\theta)\succeq \mu I$ for all $\theta$),
    then $f$ has a unique global minimizer $\theta^*$ and $\mathcal{F}(\theta)=0$ has the unique solution $\theta=\theta^*$.
\end{proposition}

\begin{proof}
    Strong convexity implies the quadratic lower bound
    $f(\theta)\ge f(0)+\langle \nabla f(0),\theta\rangle + \tfrac\mu2\|\theta\|^2$,
    hence $f(\theta)\to\infty$ as $\|\theta\|\to\infty$ (coercivity) and thus existence of a minimizer.
    Strong convexity also yields strict convexity, so the minimizer is unique and must satisfy $\nabla f(\theta^*)=0$.
\end{proof}

\begin{proposition}[Convexity + PD Hessian at the candidate minimizer]
    \label{prop:global-convex}
    Let $f$ be convex on $\mathbb{R}^p$ and suppose $\mathcal{F}(\hat\theta)=0$ with $\nabla^2 f(\hat\theta)\succ0$.
    Then $\hat\theta$ is the unique global minimizer of $f$.
\end{proposition}

\begin{proof}
    For convex $f$, every stationary point is a global minimizer.
    Assume, for contradiction, there is $\theta'\neq \hat\theta$ with $f(\theta')=f(\hat\theta)$.
    By convexity, $f$ is constant along the segment $[\hat\theta,\theta']$,
    which implies zero directional second derivative at $\hat\theta$ in the direction $d=\theta'-\hat\theta$.
    But $\nabla^2 f(\hat\theta)\succ0$ gives $d^\top \nabla^2 f(\hat\theta) d>0$ for any $d\neq0$, a contradiction.
\end{proof}

So what criteria are sufficient to exactly recover the original weights?
Suppose a training procedure (with implicit regularization) produces weights $\hat\theta$.
If we fit $\hat\Lambda$ so that \emph{stationarity at $\hat\theta$ holds}:
\[
    \nabla \mathcal{L}(\hat\theta) + \nabla \mathcal{R}(\hat\theta,\hat\Lambda) = 0,
\]
then:
\begin{itemize}
    \item By Proposition~\ref{prop:global-strong}, if $f$ is strongly convex on $\mathbb{R}^p$, retraining on $f$ recovers $\hat\theta$ uniquely.
    \item By Proposition~\ref{prop:global-convex}, the same holds if $f$ is convex and $\nabla^2 f(\hat\theta)\succ0$.
    \item By Proposition~\ref{prop:local}, even without convexity, there exists \(r>0\)
          such that \(\hat\theta\) is the unique stationary point in \(B_r(\hat\theta)\).
          Hence any retraining method that converges to a stationary point lying in
          \(B_r(\hat\theta)\) must recover \(\hat\theta\).
\end{itemize}

\paragraph{Stability under approximate gradient matching.}
\begin{proposition}[Perturbation bound under strong convexity]
    \label{prop:stab}
    Assume $f$ is $\mu$-strongly convex and let $\theta^*=\arg\min f$.
    If $\|\mathcal{F}(\hat\theta)\|\le \varepsilon$, then
    \[
        \|\hat\theta-\theta^*\| \le \frac{\varepsilon}{\mu}
        \quad\text{and}\quad
        f(\hat\theta)-f(\theta^*) \le \frac{\varepsilon^2}{2\mu}.
    \]
\end{proposition}

\begin{proof}
    Strong convexity implies \emph{strong monotonicity} of the gradient:
    $\langle \mathcal{F}(\hat\theta)-\mathcal{F}(\theta^*),\,\hat\theta-\theta^* \rangle \ge \mu \|\hat\theta-\theta^*\|^2$,
    with $\mathcal{F}(\theta^*)=0$. By Cauchy–Schwarz,
    $\mu \|\hat\theta-\theta^*\|^2 \le \|\mathcal{F}(\hat\theta)\|\,\|\hat\theta-\theta^*\|$,
    giving the distance bound.
    For the function gap, strong convexity gives
    $f(\theta^*) \ge f(\hat\theta) + \langle \mathcal{F}(\hat\theta),\theta^*-\hat\theta\rangle + \tfrac{\mu}{2}\|\theta^*-\hat\theta\|^2$.
    Rearrange and apply Young's inequality
    $\langle a,b\rangle \le \tfrac{1}{2\mu}\|a\|^2 + \tfrac{\mu}{2}\|b\|^2$ with $a=\mathcal{F}(\hat\theta)$, $b=\hat\theta-\theta^*$,
    to obtain $f(\hat\theta)-f(\theta^*) \le \|\mathcal{F}(\hat\theta)\|^2/(2\mu)$.
\end{proof}

\subsection*{Linear regression with quadratic penalty (specialization)}
Let
\[
    f(\theta)=\tfrac12\|y-X\theta\|_2^2 + \tfrac12\,\theta^\top \hat\Lambda \theta,
    \quad
    \mathcal{F}(\theta)=(X^\top X+\hat\Lambda)\theta - X^\top y,
\]
with $X\in\mathbb{R}^{n\times p}$, $y\in\mathbb{R}^n$, and assume $\hat\Lambda$ is symmetric
(only its symmetric part matters in $\nabla f$).

\begin{proposition}[Uniqueness and equality of weights]
    \label{prop:lq}
    If $A:=X^\top X+\hat\Lambda \succ 0$ and $A\hat\theta=X^\top y$, then:
    (i) $f$ is strongly convex with unique minimizer $\theta^*=A^{-1}X^\top y$;
    (ii) $\theta^*=\hat\theta$ (explicit retraining exactly recovers the original weights).
\end{proposition}

\begin{proof}
    $A\succ0$ implies $\nabla^2 f(\theta)=A\succ0$, hence strong convexity and uniqueness.
    The normal equations for the minimizer are $A\theta^*=X^\top y$, so $\theta^*=\hat\theta$ by the assumed identity.
\end{proof}

\paragraph{Remark (singular case).}
If \(A\succeq0\) but singular, the problem is convex but not strongly convex.
A minimizer exists iff \(X^\top y\in \mathrm{range}(A)\), in which case the set
of minimizers is
\[
    A^\dagger X^\top y+\ker(A).
\]
If \(X^\top y\notin \mathrm{range}(A)\), the objective is unbounded below.
Thus unique recovery is impossible without further structure (e.g., an extra
constraint or a selection rule such as minimum norm).

\paragraph{Summary.} Suppose the trained weights
\(\hat\theta\) were originally obtained as the solution to
\[
\hat\theta
\in
\arg\min_\theta
\left\{
\mathcal L(\theta,D)
+
\mathcal R(\theta,\Lambda^\star)
\right\},
\]
for some true, possibly implicit, regularizer \(\Lambda^\star\). Now consider an estimated regularizer \(\hat\Lambda\).
If \(\hat\Lambda\) satisfies the first-order matching condition
\[
\nabla_\theta \mathcal R(\hat\theta,\hat\Lambda)
=
-\nabla_\theta \mathcal L(\hat\theta,D),
\]
then \(\hat\theta\) is stationary for the regularized objective
\[
\mathcal L(\theta,D)
+
\mathcal R(\theta,\hat\Lambda), \qquad
\nabla_\theta
\left[
\mathcal L(\hat\theta,D)
+
\mathcal R(\hat\theta,\hat\Lambda)
\right]
=0.
\]

However, this first-order condition alone does not imply that retraining with
\(\hat\Lambda\) recovers \(\hat\theta\). One also needs additional curvature or
convexity assumptions. Under such assumptions, \(\hat\theta\) is a local or
global minimizer of
\[
\theta
\mapsto
\mathcal L(\theta,D)
+
\mathcal R(\theta,\hat\Lambda).
\]
If this minimizer is unique, then
\[
\hat\theta
=
\arg\min_\theta
\left\{
\mathcal L(\theta,D)
+
\mathcal R(\theta,\hat\Lambda)
\right\}.
\]

This establishes recovery of \(\hat\theta\) conditional on the
estimated regularizer \(\hat\Lambda\). It does not necessarily establish
recovery or identification of the original regularizer \(\Lambda^\star\).
In particular, it is possible that
$\hat\Lambda \neq \Lambda^\star$
while both regularizers induce the same recovered parameter \(\hat\theta\),
provided they generate the same regularizer gradient at \(\hat\theta\):
\[
\nabla_\theta \mathcal R(\hat\theta,\hat\Lambda)
=
\nabla_\theta \mathcal R(\hat\theta,\Lambda^\star).
\]
More generally, \(\hat\Lambda\) and \(\Lambda^\star\) need only be equivalent
in the directions relevant to the stationarity and optimality conditions at
\(\hat\theta\). Therefore, \(\hat\Lambda\) may rationalize or reproduce the
trained weights without being identifiable as the true regularizer.

\section{PAC Guarantees}
\label{app:pac}


\subsection{Linear parameterized case}

\paragraph{Setup.} We address the candidate regularizer case within Section \ref{sec:identifiability}, where we simultaneously fit a number of candidate regularizers with scale coefficients  -- essentially turning the problem into linear regression over a set of regularizer gradients. 

Say we have a regularizer $\gR(\theta, D, \Lambda)$ that is linear in $r$ unknown coefficients:
$$\gR(\theta, D, \Lambda) := \sum_{k=1}^r \lambda_k \gR_k(\theta, D), \qquad \Lambda = (\lambda_1,...,\lambda_r)^\top.$$

Let $b := -\nabla_\theta \gL(\hat\theta, D) \in \R^p$. We define the endpoint regularizer-gradient design matrix 
$$\Phi \in \R^{p \times r}, \qquad \Phi_{i,j} := \partial_{\theta_i} R_j(\hat\theta, D), \qquad \text{so } \nabla_\theta \gR(\theta, D,\Lambda) = \Phi \Lambda.$$
As this is a linear regression problem, the endpoint estimator is
$$\Lambda \in \argmin_{\Lambda \in \R^r} \|\Phi \Lambda - b\|_2^2.$$
The following set of assumptions suffice for a pseudo-parametric rate from using the least-squares estimator in the case where the model parameter errors are correlated.

\begin{aspt}[Endpoint well-specification]
There exists $\Lambda^\star \in \R^r$ so that $b = \Phi \Lambda^\star + \xi + a$, where $a \in \R^p$ is a deterministic approximation bias term.
    \begin{itemize}
        \item (Sub-Gaussian correlated errors) $\E[\xi | \Phi] = 0$, and there exists $\sigma \geq 0$ so that $\E[\exp(tv^\top \xi) | \Phi] \leq \exp(\sigma^2t^2 v^\top \Sigma v/2)$ for all $v \in \R^p$.
        \item (Realizability) $\alpha = \frac{1}{\sqrt{p}} \| (\Phi^\top\Phi)^{-1/2} \Phi^\top a\| < \infty$.
        \item (Identifiability and conditioning) There exist $\kappa, \rho$ so that $\lambda_{\min}(p^{-1}\Phi^\top\Phi) \geq \kappa > 0$, and $\rho = \lambda_{\max}((\Phi^\top\Phi)^{-1/2} \Phi^\top \Sigma \Phi (\Phi^\top\Phi)^{-1/2}) < \infty$.
    \end{itemize} 
    \label{aspt:parametric}
\end{aspt}
$\kappa$ is a minimum eigenvalue assumption, and $\rho$ describes an estimation error inflation factor induced by correlation across endpoint parameter errors. This is the largest eigenvalue of $\Sigma$ when restricted to the column space of the regularizer gradient design matrix. Observe that this is 1 when $\Sigma = I_p$.  

More explicitly, the error covariance matrix $\Sigma$ is the covariance of the endpoint gradient error $\xi$. The final inflation factor describes how the endpoint gradient errors are correlated after projection onto the span of candidate regularizer gradients. For example, if $\Sigma$ has large eigenvalues along directions that lie in the column space of $\Phi$, then the noise is aligned with the same directions that the candidate regularizers are trying to explain, and learning then becomes difficult.

On the other hand, if two regularizer gradients are highly collinear, then $\Phi^\top \Phi$ is ill-conditioned. This is encapsulated in the minimum eigenvalue condition. 

\begin{thm}[PAC Bound for the Linear Parameterized Case]
    Conditional on $\Phi$, it holds with probability at least $1-\delta$ that 
    $$\|\Lambda - \Lambda^\star\|_2 \leq \frac{2\sigma\sqrt{2\rho}}{\sqrt{\kappa}}\sqrt{\frac{r\log 5 + \log(2/\delta)}{p}} + \frac{\alpha}{\sqrt{\kappa}}.$$
\end{thm}

 Defining $p_{\operatorname{eff}} = {p}/{\rho}$ for the effective parameter count, this shows that the error scales in $O(\sqrt{r/p_{\operatorname{eff}}})$. This result illuminates a phenomenon we've seen within the experiments -- that the error increases in the number of regularizer parameters we estimate, while decreasing drastically when multiple independent endpoints are used as doubling the number of independent endpoints doubles the effective parameter count. 

\begin{proof}
    The normal equations state that $\Phi^\top (\Phi \Lambda - b) = 0$. Substituting and rearranging, we obtain $\Phi^\top (\Phi \Lambda - \Phi \Lambda^\star - \xi - a) = 0$. The conditioning assumption allows us to invert the covariance matrix, and doing so yields
    $$\Lambda - \Lambda^\star = (\Phi^\top \Phi)^{-1} \Phi^\top \xi + (\Phi^\top \Phi)^{-1} \Phi^\top a.$$

    Now, because $\lambda_{\min}(p^{-1}\Phi^\top\Phi) \geq \kappa > 0$, 
    \begin{align*}
        (\Phi^\top \Phi)^{-1} \Phi^\top \xi
        &= (\Phi^\top \Phi)^{-1/2} (\Phi^\top \Phi)^{-1/2}\Phi^\top \xi
        \leq \frac{1}{\sqrt{p\kappa}} (\Phi^\top \Phi)^{-1/2}\Phi^\top \xi.
    \end{align*}
    Take any vector $u \in \R^r$ where $\|u\|=1$. $u^\top (\Phi^\top \Phi)^{-1/2}\Phi^\top \xi$ is now a mean-zero sub-Gaussian random variable with variance proxy $\sigma^2 u^\top (\Phi^\top\Phi)^{-1/2} \Phi^\top \Sigma \Phi (\Phi^\top\Phi)^{-1/2} u \leq \sigma^2 \rho$. By the standard sub-Gaussian tail bound as well as a union bound over a $1/2$-net of the unit ball in $r$ dimensions (which has cardinality $(1+2/(1/2))^r = 5^r$), we can conclude that it holds w.p. at least $1-\delta$ that
    $$\| (\Phi^\top \Phi)^{-1/2}\Phi^\top \xi\|_2 \leq 2\sqrt{2\sigma^2 \rho(r \log 5 + \log(2/\delta))}.$$
    This implies a bound on the error term
    $$\|(\Phi^\top \Phi)^{-1} \Phi^\top \xi \|_2 \leq \frac{2\sigma\sqrt{2\rho}}{\sqrt{\kappa}}\sqrt{\frac{r\log 5 + \log(2/\delta)}{p}}.$$
    
    It remains to handle the deterministic bias term.  We assumed that $\alpha = \frac{1}{\sqrt{p}} \| (\Phi^\top\Phi)^{-1/2} \Phi^\top a\| < \infty$, so
    $$\|(\Phi^\top \Phi)^{-1} \Phi^\top a\|_2 \leq \|(\Phi^\top \Phi)^{-1/2} \|_{\text{op}} \|(\Phi^\top \Phi)^{-1/2} \Phi^\top a \|_2 \leq \frac{1}{\sqrt{p\kappa}} \cdot \alpha \sqrt{p},$$
    yielding the result.
\end{proof}

\subsection{Nonparametric case.}

\paragraph{Setup.} Let $b = -\nabla_\theta \gL(\hat\theta, D) \in \R^p$. Consider a regularizer class $\gR$. Defining $g_R = \nabla_\theta R(\hat\theta, D)$, this induces a gradient class $\gG = \{g_R : R \in \gR\} \subseteq \R^p$. Here, we minimize the squared error and choose $g \in \argmin_{g \in \gG} \| g - b\|_2^2$. To obtain a bound on the error, we make the following assumptions:
\begin{aspt}[Nonparametric endpoint well-specification]
    $b = b^\star + \xi$, where $b^\star$ is the predictable endpoint bias and $\xi$ is mean-zero correlated error. Additionally,
    \begin{itemize}
        \item (Sub-Gaussian correlated errors)  $\E[\xi | \Phi] = 0$, and there exists $\sigma \geq 0$ so that $\E[\exp(tv^\top \xi) | \Phi] \leq \exp(\sigma^2t^2 v^\top \Sigma v/2)$ for all $v \in \R^p$.
        \item (Bounded log-covering number) $\log \gN_{\gG}(\epsilon)$ is finite for some $\epsilon > 0$.
        \item (Correlation inflation factor) The restricted correlation inflation factor
        $$\rho(\gG) = \sup_{g \neq g' \in \gG} \frac{(g - g')^\top \Sigma (g - g')}{(g - g')^\top(g - g')} < \infty.$$
    \end{itemize}
\end{aspt}

\begin{thm}[PAC Bound for the Nonparametric Case]
    Let $\hat R$ be the empirical risk minimizer, $\gG$ the induced regularizer gradient class by the regularizer class $\gR$, and $\gG_\epsilon$ an $\epsilon$-net of $\gG$. 
    $$\frac{1}{p}\|\nabla_\theta \hat R(\hat\theta, D) + \E\nabla_\theta \gL(\hat\theta, D)\|_2^2 - \inf_{R \in \gR} \frac{1}{p}\|\nabla_\theta  R(\hat\theta, D) + \E\nabla_\theta \gL(\hat\theta, D)\|_2^2$$
    is bounded w.p. at least $1-\delta$ by
    \[\begin{aligned}
    &2 \epsilon \inf_{R \in \gR} \|\nabla_\theta  R(\hat\theta, D) + \E\nabla_\theta \gL(\hat\theta, D)\|_2 + \epsilon^2\\
    &\quad + 2\sigma \left( \sup_{g,g' \in \gG} \frac{1}{\sqrt{p}} \|g - g'\|_2\right) \sqrt{\frac{2\rho(\gG)(\log \gN_{\gG}(\epsilon) + \log(2/\delta))}{p}}.
    \end{aligned}\]
\end{thm}
As before, the effective parameter count is $p / \rho(\gG)$. This blows up if the error covariance matrix $\Sigma$ has large eigenvalues in directions representable by differences in the gradient class. Note that any blowup is not caused by the complexity of $\gG$ alone, but is instead caused by the interaction between differences in $\gG$ and the error covariance matrix.  
The error covariance matrix describes how noise in the loss gradient target is correlated across model parameters.

\begin{proof}
    Consider an $\epsilon$-net of $\gG$ that we will call $\gG_\epsilon$. Let $$g_\epsilon \in \argmin_{g \in \gG_\epsilon} \|g - b\|_2^2, \qquad g_\epsilon^\star \in \argmin_{g \in \gG_\epsilon} \|g - b^\star\|_2^2.$$
    Now because $\|g_\epsilon - b\|_2^2 \leq \|g_\epsilon^\star - b\|_2^2$ and $b = b^\star + \xi$, 
    $$\frac{1}{p}\|g_\epsilon - b^\star\|_2^2 - \frac{1}{p}\|g_\epsilon^\star - b^\star\|_2^2 \leq \frac{2}{p}\langle g_\epsilon - g_\epsilon^\star, \xi \rangle \leq 2 \sup_{g \in \gG_\epsilon} \left|  \frac{1}{p}\langle g_\epsilon - g_\epsilon^\star, \xi \rangle \right|.$$
    Now observe that $\frac{1}{p}\langle g_\epsilon - g_\epsilon^\star, \xi \rangle$ is sub-Gaussian with variance proxy bounded by
    $$\frac{\sigma^2}{p^2}( g_\epsilon - g_\epsilon^\star )^\top \Sigma ( g_\epsilon - g_\epsilon^\star ) \leq \frac{\sigma^2 \rho(\gG_\epsilon)}{p^2}\sup_{g,g' \in \gG_\epsilon}  \|g - g'\|_2.$$
    It then remains to employ the sub-Gaussian tail bound and a union bound over $\gG_\epsilon$ to find that it holds w.p. at least $1-\delta$ that 
    $$\sup_{g \in \gG_\epsilon} \left|  \frac{1}{p}\langle g_\epsilon - g_\epsilon^\star, \xi \rangle \right| \leq \sigma \sup_{g,g' \in \gG_\epsilon} \frac{1}{\sqrt{p}} \|g - g'\|_2 \sqrt{\frac{2\rho(\gG_\epsilon)(\log \gN_{\gG}(\epsilon) + \log(2/\delta))}{p}}.$$
    It then follows that $\frac{1}{p}\|g_\epsilon - b^\star\|_2^2 - \frac{1}{p}\|g_\epsilon^\star - b\|_2^2$ is bounded w.p. at least $1-\delta$ by
    $$2 \epsilon \inf_{g \in \gG}\|g - b^\star\| + \epsilon^2 + 2\sigma  \sup_{g,g' \in \gG_\epsilon} \frac{1}{\sqrt{p}} \|g - g'\|_2 \sqrt{\frac{2\rho(\gG_\epsilon)(\log \gN_{\gG}(\epsilon) + \log(2/\delta))}{p}}.$$
\end{proof}

\section{Learning a regularizer on a single trajectory}
\label{app:single-trajectory}
This generalizes the setup within the manuscript to multiple points in a single trajectory. It turns out that you can get concentration via a martingale bound, but to an object that amounts to the regularizer that best approximates the predictable bias in $T$ iterations. 

\paragraph{Setup.} Let $\hat\theta^{(1)},...\hat\theta^{(T)}$ be a single trajectory produced by the implicitly biased learning algorithm, and $\theta^{(t)},...,\theta^{(T)}$ be a coupled trajectory produced by a learning algorithm without said bias. Write $D_t \subseteq D$ for the minibatch at time $t$. Define targets
$$b_t := \frac{\hat\theta^{(t)} - \hat\theta^{(t+1)}}{\eta_t} - \frac{\theta^{(t)} - \theta^{(t+1)}}{\eta_t},$$
which we obtain by rearranging for example $\theta^{(t+1)} = \theta^{(t)} - \eta_t g_t$, where $g_t$ is the update direction.

\paragraph{Examples.} This encompasses a host of setups. For example, if one is comparing SGD to batch GD, we obtain
$$b_t = \nabla_\theta \mathcal{L}(\hat\theta^{(t)}, D_t) - \nabla_\theta \mathcal{L}(\theta^{(t)}, D).$$
If one is comparing SGD with dropout to plain SGD, we obtain
$$b_t = \frac{\hat\theta^{(t)} - \hat\theta^{(t+1)}}{\eta_t} - \nabla_\theta \mathcal{L}(\hat\theta^{(t)}, D_t).$$
And finally, if we only consider the endpoint $t=T$, we recover the formulation from earlier in the paper:
$$b_T = - \nabla_\theta \mathcal{L}(\theta^{(t)}, D).$$

\paragraph{Regression.} As before, we regress $b_1,...,b_T$ on $\nabla_\theta R(\hat\theta^{(t)}, D, \Lambda)$, optimizing over either $R \in \mathcal{R}$ or $\Lambda$ depending on whether one deals with a nonparametric regularizer class or a parametric regularizer class.

Consider the filtration $\gF_t$. Write
\[
    b_t^\star := \E[b_t \mid \gF_t], \qquad M_{t+1} := b_t - b_t^\star.
\]
Then it follows that
\[
    b_t = b_t^\star + M_{t+1}, \qquad \E[M_{t+1} \mid \gF_t] = 0.
\]

\paragraph{Nonparametric case.} We consider the nonparametric case. Define a regularizer class $\gR$, with $\rho$-log-covering number $\log \gN_{\gR}(\rho)$. To combat heteroskedasticity, we accommodate p.s.d. weight matrices $W_t$ that are $\gF_t$-measurable. Ideally we have $W_t \approx \E[M_{t+1}M_{t+1}^\top \mid \gF_t]^{-1} =: \Sigma_t^{-1}$. This induces the loss
$$\hat\gE_T(R) := \sum_{t=0}^{T-1} \|W_t^{1/2}(b_t - \nabla_\theta R(\hat\theta^{(t)}, D))\|_2^2, \text{ and we naturally select } \hat R_T \in \argmin_{R \in \gR} \hat\gE_T(R).$$
Their population equivalents are
$$\gE_T(R) := \sum_{t=0}^{T-1} \|W_t^{1/2}(b_t^\star - \nabla_\theta R(\hat\theta^{(t)}, D))\|_2^2, \text{ and } R_T^\star \in \argmin_{R \in \gR} \gE_T(R),$$
yielding the best approximation within $\gR$ to the predictable bias $b_t^\star := \E[b_t \mid \gF_t]$.

We also require
\begin{enumerate}
    \item \textbf{Identifiability.} $\forall R, R' \in \gR \;\; \exists t \in [T] \text{ s.t. } \nabla_\theta R(\hat\theta^{(t)}, D) \neq \nabla_\theta R'(\hat\theta^{(t)}, D).$
    \item \textbf{Measurability.} The maps $R \mapsto \nabla_\theta R(\hat\theta^{(t)}, D) $ and weight matrices $W_t$ are measurable for each $t$, and $b_t \in L^2$.
    \item \textbf{Bounded increments.} Define $\Delta_t(R, R') := W_t^{1/2}(\nabla_\theta R(\hat\theta^{(t)}, D) - \nabla_\theta R'(\hat\theta^{(t)}, D))$, and the martingale increments $X_{t+1}(R, R') = \langle \Delta_t(R, R'), W_t^{1/2}M_{t+1}\rangle$. Assume these increments are almost surely bounded: $|X_{t+1}(R, R')|\leq c(R, R')$ for all $t=1,...,T$.
\end{enumerate}
It is also helpful to define the quadratic variation $V_T(R, R') := \sum_{t=0}^{T-1} \E[X_{t+1}(R,R')^2 \mid \gF_t]$. Then, we can obtain the following generalization bound:

\begin{thm}[Within-trajectory generalization bound] With probability at least $1-\delta$,
    \[\begin{aligned}
    \gE_T(\hat R_T) - \inf_{R \in \gR} \gE_T(R)
    &\leq \sup_{R,R' \in \gR} O\Bigl(\sqrt{V_T(R,R') \log \left(\frac{\gN_{\gR}(1/T)}{\delta}\right)}\\
    &\qquad + c(R,R') \log \left(\frac{\gN_{\gR}(1/T)}{\delta}\right)\Bigr).
    \end{aligned}\]
\end{thm}

\begin{proof}
    We first write
    $$\hat\gE_T(R) = \sum_{t=0}^{T-1}\| W^{1/2}b_t^\star + W^{1/2}M_{t+1} - W^{1/2} \nabla_\theta R(\hat\theta^{(t)}, D))\|_2^2.$$
    Then, we can expand the square to find that
    $$\hat\gE_T(R) = \gE_T(R) + 2 \sum_{t=0}^{T-1}\langle W^{1/2}b_t^\star  -  W^{1/2} \nabla_\theta R(\hat\theta^{(t)}, D)), W^{1/2}M_{t+1} \rangle + \sum_{t=0}^{T-1}\| W^{1/2}M_{t+1}\|_2^2.$$
    Continuing the argument, we note that $\hat \gE_T(\hat R_T) = \inf_{R \in \gR} \hat \gE_T(R)\leq \hat\gE_T(R)$, so it must hold that
    \begin{align*}
         & \gE_T(\hat R_T) + 2 \sum_{t=0}^{T-1}\langle W^{1/2}b_t^\star  -  W^{1/2} \nabla_\theta \hat R_T(\hat\theta^{(t)}, D)), W^{1/2}M_{t+1} \rangle + \sum_{t=0}^{T-1}\| W^{1/2}M_{t+1}\|_2^2 \\
         & \leq \gE_T(R) + 2 \sum_{t=0}^{T-1}\langle W^{1/2}b_t^\star  -  W^{1/2} \nabla_\theta R(\hat\theta^{(t)}, D)), W^{1/2}M_{t+1} \rangle + \sum_{t=0}^{T-1}\| W^{1/2}M_{t+1}\|_2^2.
    \end{align*}
    Cancelling and rearranging terms, we arrive at an expression for $\gE_T(\hat R_T) - \gE_T(R)$ that amounts to
    \begin{align*}
        \gE_T(\hat R_T) - \gE_T(R)
         & = 2 \sum_{t=0}^{T-1}\langle W^{1/2} \nabla_\theta \hat R_T(\hat\theta^{(t)}, D)) -  W^{1/2} \nabla_\theta R(\hat\theta^{(t)}, D)), W^{1/2}M_{t+1} \rangle \\
         &\leq 2\sup_{R,R'\in \gR} \left|\sum_{t=0}^{T-1} X_{t+1}(R, R')\right|.
    \end{align*}
    Now, note that $\E[M_{t+1} \mid \gF_t]=0$ implies that $\E[X_{t+1}(R,R') \mid \gF_t] = 0$ is a martingale increment, with quadratic variation $V_T(R,R')$ and almost surely bounded increments $|X_{t+1}(R, R')|\leq c(R, R')$. We can then apply Freedman's inequality to find
    $$\prob\left(\left|\sum_{t=0}^{T-1}X_{t+1}(R,R')\right| \geq \sqrt{2 V_T(R,R') \epsilon} + c(R,R') \epsilon/3\right) \leq 2\exp(-\epsilon).$$
    Applying a union bound over a $1/T$-net of $\gR$ then yields our desired result:
    \[\begin{aligned}
    \gE_T(\hat R_T) - \inf_{R \in \gR} \gE_T(R)
    &\leq \sup_{R,R' \in \gR} O\Bigl(\sqrt{V_T(R,R') \log \left(\frac{\gN_{\gR}(t/T)}{\delta}\right)}\\
    &\qquad + c(R,R') \log \left(\frac{\gN_{\gR}(1/T)}{\delta}\right)\Bigr).
    \end{aligned}\]
\end{proof}


\section{Computing loss and regularizer gradients}
\label{app:computing-grads}
\paragraph{Computational complexity.}


Often, the loss is of the form $\mathcal{L}(\theta,D) = \mathcal{L}(\hat{Y}(\theta),D)$, where $\hat{Y}(\theta)=\hat{Y}(\theta,D)$ is the output of the model with weights $\theta$.
In that case, Equation \ref{eq:grad-matching} becomes
\[ \nabla_\theta \mathcal{R}({\theta},D)|_{\hat\theta(D)}^\top
    = -\left(\frac{\partial \mathcal{L}(\hat{Y},D)}{\partial\hat{Y}}\right)^\top
    \nabla_\theta{\hat{Y}}({\theta})|_{\hat\theta(D)},\]
where $\frac{\partial \mathcal{L}}{\partial\hat{Y}}$ is the partial derivative of the empirical loss with respect to the model output/predictions, and $\nabla_\theta{\hat{Y}}({\theta})$ is the Jacobian of the model output with respect to the model's weights.

For example taking the empirical loss to be mean-squared error:
$\mathcal{L}(\theta,D)= \frac{1}{2} \|{Y} - {\hat{Y}}({\theta})\|^2$, we obtain
\begin{align*}
    \nabla_\theta \mathcal{R}({\theta},D) |_{\hat\theta(D)} = ({Y} - {\hat{Y}}({\hat{\theta}}))^\top  \nabla_\theta{\hat{Y}}({\theta})|_{\hat\theta(D)}.
\end{align*}
This is just the product of the model's residuals (${Y} - {\hat{Y}}$) and its Jacobian evaluated at the learned weights ($\nabla_\theta{\hat{Y}}({\theta})|_{\hat\theta(D)}$), both of which can be readily computed.

It is typically easy to derive and compute $\frac{\partial \mathcal{L}}{\partial\hat{ y}}$ for standard loss functions, which means we can compute the vector-Jacobian product $(\frac{\partial \mathcal{L}}{\partial\hat{ y}})^\top  \nabla_\theta{\hat{Y}}({\theta})$ instead of calculating and storing the Jacobian itself. This vector-Jacobian product has dimensionality $1\times p$, or equivalently $p$ if gradients are treated as column vectors. Computing it avoids explicitly  forming the full Jacobian $\nabla_\theta \hat Y(\theta)\in\mathbb R^{k\times p}$, where $k$ is output dimension of the model. Moreover, this computation doesn't depend on the choice of $\mathcal{R}$ at all -- this will be helpful when we iterate over candidate regularizers, as this quantity need not be recomputed.

As mentioned in the main text, our stationarity condition implies a system of $p$ equations where $p$ is the dimensionality of $\theta$. To see this, note that $\nabla_\theta \mathcal{R}({\theta},D) \in \R^p$, $\left(\tfrac{\partial \mathcal{L}(\hat{y},D)}{\partial\hat{ y}}\right)^\top \in \R^k$ where $k$ is the output size of the model, and $\nabla_\theta{\hat{y}}({\theta}) \in \R^{k \times p}$.

\section{Related work on estimating inductive bias}
\label{app:related-work}
\cite{boopathyExactComputationInductive2024} propose a Bayesian framework for quantifying the \emph{amount} of inductive bias needed to achieve generalization on some prediction task.
They define inductive bias of a task as the negative log probability that a hypothesis $h$ achieves some test error rate $\varepsilon$ -- intuitively, if a randomly sampled hypothesis $h\sim p_h$ is unlikely (small probability) then a large inductive bias is required to reliably find a generalizing hypothesis (large negative log-probability).

\cite{immerProbingQuantifyingInductive2022} propose a method of quantifying inductive bias based on probing intermediate representations.
They consider Bayesian evidence as a proxy for inductive bias, formalizing it as the maximum evidence (how likely it is that a particular dataset could have been generated by a given model) for some intermediate representation, over all possible probes in a function class (e.g. linear probes).
The intuition is that overly-simple models are unlikely because they don't explain the data well and complex models are unlikely because they are so high-dimensional that many candidate models can also explain the data equally well.

Neither of these approaches involves actually estimating the parameters of the kinds of implicit regularizers suggested in the literature, as we do in the present work. Rather they quantify inductive bias as a scalar quantity and try to judge the amount of inductive bias present in various models or representations. In these settings, any predictive model has a fixed amount of inductive bias on a prediction task -- there is no way to evaluate competing hypotheses for how that bias is produced.
By directly estimating the parameters of hypothesized regularizers, our framework enables empirically analyzing and contrasting theoretical predictions about implicit bias in realistic models.

\section{Bootstrapping}
\label{app:bootstrapping}

Bootstrap resampling is used only in the experiment in this appendix, not in Figures~\ref{fig:elasticnet-recovery}--\ref{fig:igr}.


\begin{figure}[h!]
    \centering
    \includegraphics[width=\linewidth]{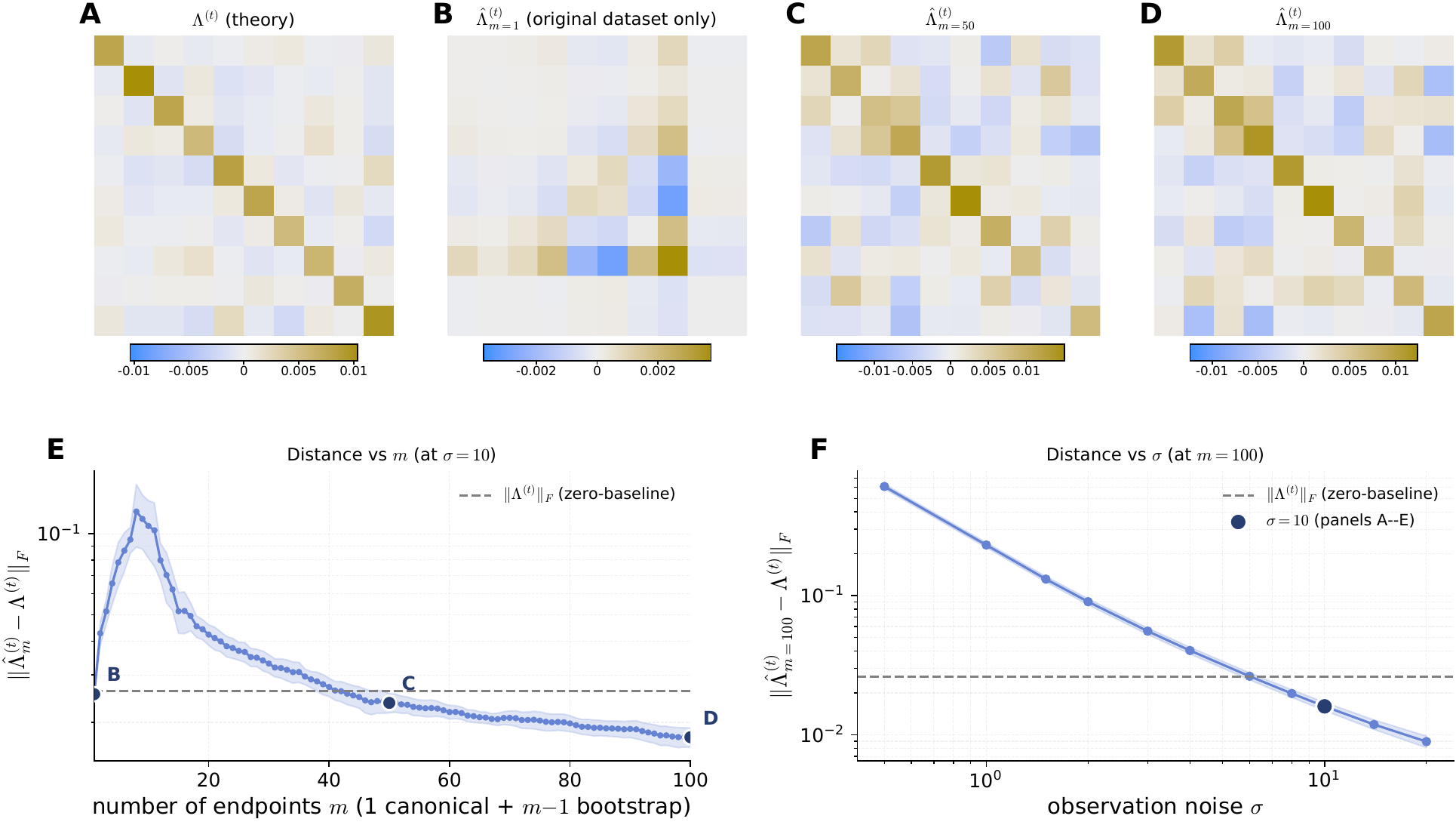}
    \caption{\textbf{Bootstrap recovery of the early-stopped GD regularizer.}
    Parallel to Figure~\ref{fig:ols-early-stopping} but with bootstrap resamples of a single training set in place of independent ground-truth weight redraws.  We use a noisier variant of the OLS DGP ($\sigma=10$ rather than $\sigma=1$ in the main composite figure): bootstrap recovery requires sufficient OLS sampling variance ($\propto \sigma/\sqrt{n}$) for the bootstrap-induced spread in $\hat\theta$ to span the $p(p+1)/2 = 55$-dim space of symmetric matrices.
    \textbf{(A)} Theoretical $\Lambda^{(t)}$ at the canonical fixed iterate $t=500$.
    \textbf{(B)} The minimum-norm fit from the single canonical training endpoint.  With only $p$ stationarity equations in $p(p+1)/2$ unknowns, the LS system is underdetermined and \texttt{lstsq} returns the smallest-Frobenius-norm $\hat\Lambda$ satisfying the constraints; most of the matrix is left at zero.
    \textbf{(C)} Recovery from $m=50$ endpoints (1 canonical + 49 bootstrap).
    \textbf{(D)} Recovery from $m=100$ endpoints, comfortably past the rank threshold $m \approx p$ where the LS design matrix transitions from underdetermined to overdetermined.
    \textbf{(E)} Distance to theory $\|\hat\Lambda_m^{(t)} - \Lambda^{(t)}\|_F$ as $m$ grows from $1$ to $100$ at $\sigma=10$, with markers at the $m$ values used in panels B--D.  The dashed gray line is $\|\Lambda^{(t)}\|_F$, equivalent to the trivial baseline $\hat\Lambda=0$. Band shows mean and 95\% CI across $10$ pools of bootstrap resamples.
    \textbf{(F)} Distance at $m=100$ swept over $\sigma$ (mean and 95\% CI across $10$ pools of bootstrap resamples).  Below $\sigma\approx 6$ bootstrap recovery is worse than the zero-baseline; above it the recovered matrix is increasingly informative.}
    \label{fig:ols-bootstrap-composite}
\end{figure}

In Figure~\ref{fig:ols-bootstrap-composite} we replicate the OLS early-stopping recovery experiment (Figure~\ref{fig:ols-early-stopping}) using bootstrap resampling of a single training set rather than independent draws of ground-truth weights.  Each bootstrap endpoint is trained from scratch on a row-resampled $(X^{*(b)},y^{*(b)})$ for $b=1,\dots,m-1$, with the canonical training endpoint included as $b=0$.  Compared to the multi-dataset experiment, bootstrap recovery is qualitatively weaker: the multi-dataset fixed-$t$ system is exactly consistent (all endpoints share the same $X$ and hence the same theoretical $\Lambda^{(t)}$), so it reaches numerical recovery at $m=p$, whereas the bootstrap system is only approximately consistent ($\Lambda^{*(b),(t)}$ varies with the resampled $X^{*(b)}$) and exhibits a residual model-misspecification floor.  Recovery is also gated on observation noise: the spread in $\hat\theta^{*(b)}$ scales as $\sigma/\sqrt n$, so for very clean data (small $\sigma$) bootstrap thetas cluster too tightly around $\hat\theta(D)$ for the LS system to be identifiable.  The $\sigma$ sweep in Panel F shows the practical regime: bootstrap recovery beats reporting zero only when the OLS sampling variance is non-negligible relative to the coefficient magnitude.

\begin{algorithm}[h!]
    \caption{Bootstrap-augmented estimation}
    \label{alg:bootstrap-ber}
    \begin{algorithmic}[1]
        \Require Dataset \(D\), learning algorithm \(\mathcal{A}\), regularizer family \(\mathcal{R}(\theta,D,\Lambda)\), number of bootstrap replicates \(B\)
        \Ensure Estimate \(\hat{\Lambda}\)

        \State Train the predictive model on \(D\) to obtain \(\hat{\theta}(D)\)

        \For{\(b=1,\dots,B\)}
        \State Sample a bootstrap sample \(D^{*(b)}\) from \(D\) with replacement
        \State Retrain the predictive model on \(D^{*(b)}\) to obtain \(\hat{\theta}^{*(b)}\)
        \EndFor

        \State Estimate \(\hat{\Lambda}\) by minimizing
        \[\begin{aligned}
            &
            \left\|
            \nabla_\theta \mathcal{R}(\theta,D,\Lambda)\big|_{\hat{\theta}(D)}
            +
            \nabla_\theta \mathcal{L}(\theta,D)\big|_{\hat{\theta}(D)}
            \right\|_2^2
            \\
            &+
            \sum_{b=1}^{B}
            \left\|
            \nabla_\theta \mathcal{R}(\theta,D^{*(b)},\Lambda)\big|_{\hat{\theta}^{*(b)}}
            +
            \nabla_\theta \mathcal{L}(\theta,D^{*(b)})\big|_{\hat{\theta}^{*(b)}}
            \right\|_2^2
        \end{aligned}\]

        \State \Return \(\hat{\Lambda}\)
    \end{algorithmic}
\end{algorithm}

\FloatBarrier

\section{Additional considerations for stable estimation of implicit regularization}

\subsection{Normalizing gradients} When candidate regularizers induce gradients with very different norms, we use a normalized gradient-matching variant that first rescales each candidate gradient to unit norm, fits coefficients in that normalized basis, and then maps back to the original regularization parameters.

\subsection{Evaluating goodness of fit}
\label{app:gradient-fit}
Evaluating quality of fit is central to the utility of the framework. Equation~\ref{eq:gradient-fit} compares the squared residual after adding the fitted regularizer gradient with the squared loss gradient before adding it. It is analogous to the uncentered $R^2$ in linear regression: a value near one indicates that the candidate family closely matches the target gradient, zero indicates no improvement over the zero regularizer, and a negative value indicates a larger residual. If the family includes the zero regularizer and the gradient-matching objective is minimized exactly on the same target, the score cannot be negative.

Both gradients should be evaluated at the same weights and on the same dataset, using the fitted regularizer in its original units. To assess fit to the empirical loss gradient, use the full dataset or a sufficiently accurate average of minibatch gradients before computing the squared norms. Averaging scores computed separately on minibatches is a different diagnostic and can be dominated by minibatch variability. When the loss gradient is nearly zero, the ratio becomes sensitive to small numerical or sampling errors; report the loss-gradient and residual norms alongside the score, and treat the score as undefined when its denominator is zero.

A low score may indicate misspecification, but it can also arise from noisy gradient estimates or failure to optimize the regularizer coefficients. These possibilities should be checked before rejecting the candidate family. Moreover, an endpoint that is still moving under its training procedure may reflect incomplete optimization as well as implicit regularization. A high score establishes a good gradient match at the evaluated weights, but does not identify a unique regularizer or show that it captures the effect of a particular training mechanism. Controlled ablations and retraining with the fitted explicit penalty provide stronger checks, as discussed in Section~\ref{sec:eval}.

\subsection{Sensitivity to training duration and convergence}
\label{app:training-duration}
The estimated regularizer can depend on when we evaluate the predictive model during training. Fitting at successive checkpoints helps assess this sensitivity: changes in the estimate may reflect evolving implicit effects as well as incomplete optimization of the predictive model. This is distinct from error in fitting the regularizer coefficients at fixed weights.

Figure~\ref{fig:dropout-checkpoint-stability} illustrates this for the dropout experiment in Section~\ref{sec:dropout} and Figure~\ref{fig:dropout-l2}. We use the ten depth-3, width-256 runs with dropout rate $0.3$ from the final grid, following the training and stopping protocol in Appendix~\ref{app:dropout-methods}. At each saved checkpoint we compute the exact scalar ridge fit to the full clean-training loss gradient. The fitted coefficients decrease substantially early in training and change less in absolute magnitude later on, reaching small positive values at the stopping endpoints. Because these coefficients approach zero, small absolute changes can still correspond to appreciable relative changes. This comparison characterizes sensitivity to training duration; coefficient stability alone does not establish stationarity or adequacy of the candidate family.

\begin{figure}[htbp]
    \centering
    \includegraphics[width=0.9\linewidth]{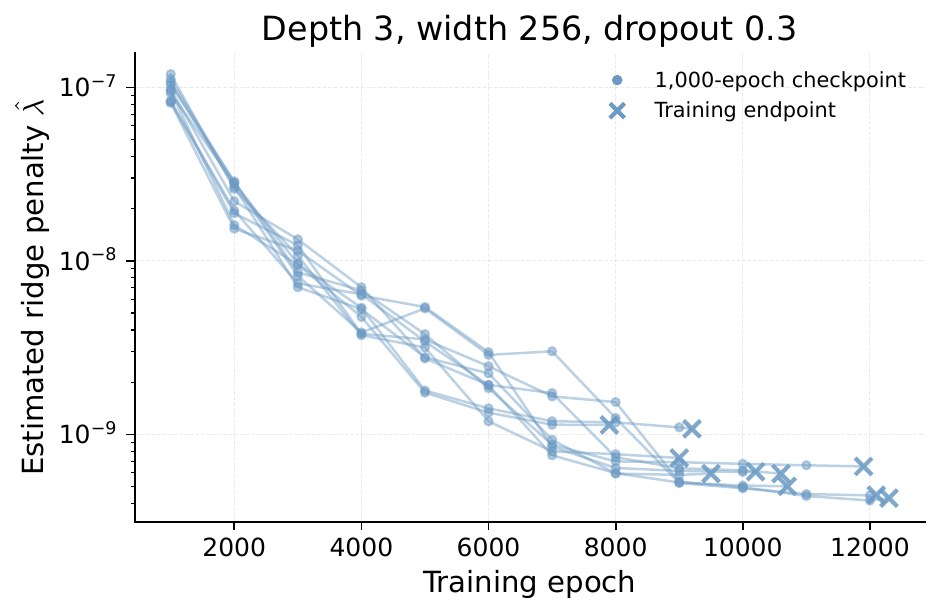}
    \caption{\textbf{Sensitivity of the estimated dropout ridge coefficient to training duration.} Each translucent curve represents one of ten seeds. Circles mark saved checkpoints every 1,000 epochs; crosses mark each run's training-loss stopping endpoint (7,900--12,300 epochs). The vertical axis is logarithmic. Estimates are computed independently at each checkpoint, without normalizing by the final epoch or coefficient. All signed coefficients shown are positive and therefore equal their nonnegative projections. Pilot extensions beyond the stopping endpoint are excluded.}
    \label{fig:dropout-checkpoint-stability}
\end{figure}

\section{Further experiment details}

\label{app:experiment-details}

\paragraph{Code and data availability.}
The implementation and reproduction workflow are publicly available at
\url{https://github.com/jrudoler/estimating-implicit-regularization}.
The repository documents the software environment, dependency versions, and
commands for training, analysis, and figure generation. Synthetic data are
generated by the experiment scripts; MNIST is downloaded through TorchVision.

\paragraph{Existing assets and licenses.}
We use MNIST \citep{lecun2010mnist}, created by Yann LeCun, Corinna Cortes,
and Christopher J.~C.~Burges. Its dataset card lists the MIT license
(\url{https://huggingface.co/datasets/ylecun/mnist}).
PyTorch and TorchVision are distributed under BSD 3-Clause licenses
(\url{https://github.com/pytorch/pytorch/blob/main/LICENSE} and
\url{https://github.com/pytorch/vision/blob/main/LICENSE}, respectively).
The reproduction workflow downloads MNIST from the upstream source rather
than redistributing it as a new dataset.

\subsection{Illustrative figures (Figs.~\ref{fig:tradeoff-vis} and \ref{fig:gradient-residuals})}
Figures~\ref{fig:tradeoff-vis} and \ref{fig:gradient-residuals} are programmatic illustrations rather than empirical experiments. Figure~\ref{fig:tradeoff-vis} illustrates ridge regression on a scalar linear model $y = \theta x + \varepsilon$ with $\varepsilon \sim \mathcal{N}(0, \sigma^2)$, $\sigma = 1.8$, $n = 20$ observations, $x \sim \mathcal{N}(0,1)$, true $\theta = 2.0$, and $\ell_2$ penalty strength $\lambda = 0.8$. The MSE, $\ell_2$ penalty, and total objective are plotted as functions of the scalar parameter $\theta$; gradient arrows at $\hat\theta_{\mathrm{ridge}}$ confirm the equal-and-opposite stationarity condition $\nabla_\theta \mathcal{L} + \nabla_\theta \ell_2 = 0$. Figure~\ref{fig:gradient-residuals} simulates a two-parameter $\tanh(wx+b)$ network on $96$ training points $x \in [-2.5, 2.5]$ with target $y = 0.75\,\tanh(1.4x - 0.35) + 0.12\,\sin(2.7x)$; gradient flow is approximated by full-batch GD at learning rate $4\times 10^{-3}$ for $3000$ steps, and SGD at learning rate $0.40$ with batch size $8$ for $6$ steps.

\subsection{Recovering elastic net}

We trained a non-linear neural network with one hidden layer (200 units) and ReLU activation to predict a scalar output from a 10-dimensional input. The model was trained using stochastic gradient descent with a learning rate of $10^{-3}$ for up to 200 epochs, with early stopping (patience = 15) based on training loss.

Synthetic data were generated by sampling inputs $X \in \mathbb{R}^{5000 \times 10}$ from a standard normal distribution, and computing targets as
\[
    y = X\theta + \varepsilon, \quad \theta_i \stackrel{\text{iid}}{\sim} \mathcal{N}(0, 5^2), \quad \varepsilon_j \stackrel{\text{iid}}{\sim} \mathcal{N}(0,\, 0.5^2).
\]
Data were split 80/20 into training and test sets, and the network was trained in mini-batches of size 32. The bias estimator $\smash{\hat\lambda}$ was fit by Adam at learning rate $10^{-2}$ for up to $1000$ epochs with early-stopping patience $50$ on the gradient-matching loss.
For each $(\lambda_1,\lambda_2,\beta)$ in a grid search, we repeated the pipeline on independent synthetic datasets, resampling $X$, $\beta$, $\varepsilon$, and the train/test split with distinct random seeds (ten seeds per configuration).
In Figure~\ref{fig:elasticnet-recovery} we plot the mean estimated $\smash{\hat\lambda_i}$ against the true $\lambda_i$ with $\pm 1$ standard error bars across replicates, so that variance across runs is visible alongside any systematic bias.

We first attempted elastic net regularization (combining $\ell_1$ and $\ell_2$ penalties), but it failed to recover the true parameters. This is likely due to the non-smoothness of the $\ell_1$ norm, which hinders gradient-based optimization and biases solutions toward zero.

Instead, we applied a smoothed $\ell_1$ penalty, related to the Huber loss:
\[
    \mathcal{R}(\theta, \beta) =
    \begin{cases}
        \frac{1}{2\beta} \theta^2   & \text{if } |\theta| < \beta \\
        |\theta| - \frac{1}{2}\beta & \text{otherwise}
    \end{cases},
\]
with smoothing parameter $\beta = 0.001$ and $\ell_1$ penalty weight set to $1.0$.

The gradient is
\[
    \frac{d\mathcal{R}}{d\theta}=
    \begin{cases}
        \dfrac{\theta}{\beta}, & |\theta|<\beta,   \\[0.25em]
        \mathrm{sign}(\theta), & |\theta|\ge\beta.
    \end{cases}
\]
For \(|\theta|<\beta\), the behavior is purely quadratic---identical to an \(\ell_2\) penalty of strength \(1/\beta\)---whereas for \(|\theta|\ge\beta\) it is linear (\(\ell_1\)). When the coefficient \(\lambda_1\) on this penalty is very large (and \(\beta\) small), nearly all \(\theta\) lie in the quadratic region, so
\(
\lambda_1\,\mathcal{R}(\theta)\approx\tfrac{\lambda_1}{2\beta}\,\theta^2,
\)
i.e.\ an effective \(\ell_2\) penalty of strength \( \tfrac{\lambda_1}{\beta}\), which produces the same severe under-fitting characteristic of heavy \(\ell_2\) regularization. Therefore when the penalty is large and many $\theta$ are close to zero the $\lambda_1, \lambda_2$ are not identifiable.

Figure \ref{fig:elasticnet-recovery} summarizes elastic-net recovery at $\beta = 10^{-3}$ on a $6\times 6$ grid of true $(\lambda_1,\lambda_2)$ with $10$ independent dataset draws per cell.
Each panel plots the estimated $\hat\lambda_i$ against the true $\lambda_i$ on log--log axes, with one line per value of the other true penalty and error bars giving $\pm 1$ standard error across seeds; the dashed diagonal marks perfect recovery.
For large penalty values, smoothing near zero makes estimation inaccurate because the true parameters are not identifiable when many $\theta$ are small.

\subsection{OLS early-stopping experiments}
\label{app:ols-methods}

All OLS panels in Figure~\ref{fig:ols-early-stopping} use a shared noisy data-generating process. We draw $X \in \mathbb{R}^{1000 \times 10}$ with iid standard normal entries, draw coefficients $\theta_i \stackrel{\text{iid}}{\sim}\mathcal{N}(0,3^2)$, and observe
\[
    y = X\theta + \varepsilon,
    \qquad
    \varepsilon_i \stackrel{\text{iid}}{\sim}\mathcal{N}(0,1).
\]
The observation noise means that the empirical OLS minimizer is an unbiased but sample-specific estimate of the coefficient vector, so the trained weights need not coincide exactly with the ground truth even when the effective regularizer reproduces the observed-data solution.

For the single-endpoint diagonal estimate and the scalar $\lambda_t$ trace, we use seed $56$ for the shared draw of $(X,\theta,\varepsilon)$. The diagonal estimate (Panel C) trains full-batch GD from zero for up to $500$ epochs with early-stopping patience $5$ and no minimum improvement threshold (i.e.\ any decrease counts as improvement). The scalar trace (Panel E) evaluates a single deterministic GD trajectory at checkpoints $t \in \{1,2,5,10,20,50,100,150,200,300,500,1000\}$ and fits one ridge coefficient at each checkpoint by gradient matching with Adam (learning rate $0.05$, maximum $30{,}000$ epochs, patience $2000$). The GD trajectory for Panel E is distinct from the early-stopped model in Panels A--C; it runs for exactly the number of steps needed to reach $t=1000$ without early stopping, and the bias estimator at each checkpoint is independently early-stopped.

For the full symmetric-matrix estimator, we hold the design matrix $X$ fixed and form $10$ independent endpoint pools of $100$ endpoints each, with a newly drawn coefficient vector and independent observation noise per endpoint.

\paragraph{Panel B (fixed iterate).} To ensure the stacked system $\Lambda^{(t)}\hat\theta_k = -\nabla L_k(\hat\theta_k)$ has a single well-defined theoretical target $\Lambda^{(t)}$, all endpoints in the Panel B pools are stopped at the same fixed iterate $t$ equal to the early-stopping epoch of the single-endpoint run from Panels A and C. Each endpoint runs full-batch GD from zero for up to $2000$ epochs; the iterate at the shared $t$ is extracted without early stopping.

\paragraph{Panel D (variable early stopping).} To demonstrate recovery in a more realistic and challenging setting, Panel D uses per-endpoint early stopping: each endpoint runs full-batch GD from zero for up to $2000$ epochs with patience $5$ and no minimum improvement threshold, and its individual iterate count $t_k$ is determined by when its own training loss plateaus. Because $t_k$ varies across endpoints, the theoretical target used to evaluate the distance $\|\hat{\Lambda}^{(t_k)}_m - \Lambda^{(t_k)}\|$ is the per-pool median of the endpoint-specific matrices $\Lambda^{(t_k)}$; the variance in $t_k$ across endpoints within a pool is small in practice, making this a mild approximation. The smooth decay in Panel D (with no identifiability cliff at $m=p$) reflects the richer signal in variable-stop endpoints compared to a fixed-$t$ setup.

Figure \ref{fig:ols-early-stopping}E uses the same noisy OLS data-generating process as above, and we trained the regression model using full-batch gradient descent with step size $\eta = 10^{-2}$ for the normalized squared loss, evaluating $\hat\theta^{(t)}$ at $t \in \{1, 2, 5, 10, 20, 50, 100, 150, 200, 300, 500, 1000\}$. At each $t$ we fit a single-parameter ridge bias by gradient matching using Adam (learning rate $0.05$, up to $30{,}000$ epochs, early-stopping patience $2000$). The figure shows a single trajectory (not averaged).

\subsection{Implicit bias of dropout}
\label{app:dropout-methods}
For Figure~\ref{fig:dropout-l2}, we train fresh ReLU classifiers on MNIST with dropout rates
$p \in \{0, 0.05, 0.10, 0.15, 0.20, 0.30, 0.40, 0.50\}$,
depths $d \in \{3, 5\}$ (number of hidden layers),
and widths $w \in \{128, 256, 512\}$, with ten seeds per configuration (480 models).
Dropout follows each hidden ReLU except the last; batch normalization and explicit weight decay are disabled. We retain the same fixed 54,000-example subset of the MNIST training set across runs; the remaining 6,000 examples and the test set are not used for training, stopping, or coefficient estimation.

\paragraph{Training and stopping.}
We use shuffled mini-batch SGD with momentum $0.9$, batch size $2048$, and initial learning rate $0.05$. The stopping statistic is the nonoverlapping 100-epoch mean of sample-weighted training cross-entropy, with dropout active. An improvement must exceed $10^{-7}$ in absolute loss units; the relative threshold is zero. After 500 epochs without such improvement, we reduce the learning rate tenfold, down to $0.0005$, resetting the best-loss reference and patience after each reduction. At the minimum learning rate, we stop after 1,000 epochs without improvement. The cap is 100,000 epochs; all 480 models stopped by the loss rule, at 4,800--24,200 epochs (median 10,100). No gradient diagnostic or coefficient sign enters stopping.

\paragraph{Closed-form ridge fit.}
At each stopping endpoint $\hat\theta$, let $g=\nabla_\theta\mathcal{L}_{\mathrm{clean}}(\hat\theta)$ be the mean cross-entropy gradient over all 54,000 training examples with dropout disabled. For $\mathcal{R}(\theta,\lambda)=\lambda\|\theta\|_2^2$, including weights and biases, we compute
\[
    \hat\lambda_{\mathrm{signed}}
    = -\frac{\langle g,2\hat\theta\rangle}{\|2\hat\theta\|_2^2},
    \qquad
    \hat\lambda = \max(0,\hat\lambda_{\mathrm{signed}}).
\]
Model arithmetic is float32; sample-weighted gradient accumulation and scalar products are float64. No iterative coefficient optimizer is used. Signed estimates are retained; all happened to be positive. Figure colors show $R_{\mathrm{grad}}^2$ from \eqref{eq:gradient-fit}, not raw fitting loss or prediction accuracy.

\paragraph{Pilot robustness and retained trajectories.}
Before expanding the grid, we extended seven seed-123 pilots: depth 3/width 256 at dropout $0$, $0.3$, and $0.5$, plus the four remaining architectures at dropout $0.5$. Each continued at the terminal learning rate for at least as many additional epochs as its stopping duration. The predeclared practical check requires either both signed coefficients within $[-10^{-7},10^{-7}]$ or relative change at most $5\%$. Six pairs met the near-zero criterion, with relative changes up to $17.66\%$; the remaining pair changed by $4.30\%$. This is not a uniform relative-accuracy or stationarity guarantee. The figure uses the original stopping endpoints, not the extended ones. Initial, every-1,000-epoch, stopping, and pilot-extension checkpoints are retained. Zero-dropout controls are neither removed nor baseline-subtracted; the endpoint fit can reflect residual optimization gradients even when the direct dropout contribution is zero.

\subsubsection{Validation by explicit replacement}
\label{app:dropout-replacement}
We examined whether the estimated regularizer could reproduce the behavior of a dropout-trained model when applied explicitly during training without dropout. We used the ten depth-3, width-256 MNIST classifiers trained with dropout rate $0.3$ in Figure~\ref{fig:dropout-l2}. For each seed, we restarted from that model's saved initial weights and trained without dropout, either with no explicit penalty or with the fitted ridge penalty $\hat\lambda_s\|\theta\|_2^2$, including weights and biases. Each coefficient was taken directly from its corresponding dropout endpoint, without pooling, baseline subtraction, or tuning; these coefficients ranged from $4.29\times10^{-10}$ to $1.13\times10^{-9}$.

The replacements used the same training split, initialization, mini-batch SGD settings, precision, learning-rate schedule, and training-loss stopping rule as the final dropout grid described above. The explicit penalty contributed to the parameter updates, while training cross-entropy alone controlled learning-rate reductions and stopping. Thus, the stopping criterion was shared, although the stopping epochs could differ. The no-penalty controls reproduced the original zero-dropout endpoints exactly, including their parameter values. Test data were used only for evaluation after training, not to select coefficients or stopping points.

Median test accuracy was $98.345\%$ for the dropout targets, compared with $97.86\%$ without dropout and $97.85\%$ with the fitted ridge penalty. The ridge replacements had lower test accuracy than their paired dropout targets for all ten seeds, with a mean shortfall of $0.508$ percentage points. They disagreed with the target on a median of $1.655\%$ of test predictions, although all models correctly classified every training example. Thus, the fitted scalar ridge penalty did not recover the original model's test performance or predictions in this configuration. This is a concrete example in which an estimated regularizer can summarize a component of the endpoint gradient without providing a sufficient replacement for the training procedure. Together with the low gradient-fit $R_{\mathrm{grad}}^2$, this result illustrates the limits of a scalar ridge description of dropout's complex, data-dependent effects, while leaving the fitted projection useful for interpreting variation across training configurations.

\subsection{Recovering Barrett implicit gradient regularization}
\label{app:igr-methods}

This appendix details the known-coefficient recovery experiment for the implicit gradient regularizer derived by \citet{barrettImplicitGradientRegularization2022}. The goal is not to revalidate all of \citeauthor{barrettImplicitGradientRegularization2022}'s empirical claims about generalization. Instead, we use their derivation as a calibration target: if our trajectory-gradient-matching estimator is working, it should recover the coefficient $\lambda = \eta p/4$ from observed deviations between discrete gradient descent and gradient flow.

\paragraph{Relation to Equation 4.}

The trajectory-deviation estimator in Eq. 4 compares an observed update to an unregularized reference trajectory. In the Barrett--Dherin setting, the observed update is one full-batch GD/Euler step, while the relevant unregularized reference is the continuous gradient flow of $L$ over the same elapsed time $\eta$, not the linearized displacement $-\eta \nabla L(\theta)$. Comparing directly to $-\eta \nabla L(\theta)$ would give zero residual for full-batch GD; using the gradient-flow displacement instead isolates the finite-step backward-error term that corresponds to implicit gradient regularization.

\paragraph{Background and notation.}
For a loss $\mathcal{L} : \mathbb{R}^p \to \mathbb{R}$ with gradient $g(\theta) := \nabla\mathcal{L}(\theta)$ and Hessian $H(\theta) := \nabla^2\mathcal{L}(\theta)$, discrete gradient descent with step size $\eta$ iterates $\theta_{k+1} = \theta_k - \eta\,g(\theta_k)$. \citet{barrettImplicitGradientRegularization2022} apply backward error analysis to this Euler integration of the continuous gradient flow $\dot\theta = -g(\theta)$ and show that, for sufficiently smooth $\mathcal{L}$ and sufficiently small $\eta$, the discrete iterates lie close to the \emph{modified} gradient flow
\[
    \dot\theta \;=\; -g(\theta) \;-\; \tfrac{\eta}{2}\,H(\theta)\,g(\theta) \;+\; O(\eta^2)
    \;=\; -\nabla\widetilde{\mathcal{L}}(\theta) \;+\; O(\eta^2),
    \qquad
    \widetilde{\mathcal{L}}(\theta) \;=\; \mathcal{L}(\theta) + \lambda\,R_{IG}(\theta),
\]
where $\lambda = \eta p/4$ and $R_{IG}(\theta) = \|g(\theta)\|_2^2 / p$. The implicit regularizer's gradient is $\nabla\mathcal{R}(\theta,\lambda) = (2\lambda/p)\,H(\theta)g(\theta) = (\eta/2)\,H(\theta)g(\theta)$ at Barrett's identification of $\lambda$.

\paragraph{A trajectory-based estimator for $\lambda$.}
Trained neural networks are not stationary points of the augmented loss $\widetilde{\mathcal{L}}$, so the static stationarity condition $g + \nabla\mathcal{R} = 0$ used elsewhere in this paper does not hold meaningfully along the trajectory. Instead, we estimate $\lambda$ from the trajectory-deviation form. Fix a parameter vector $\theta$ and compare a single full-batch GD step
\[
    \Delta\theta_{\text{GD}} \;=\; -\eta\, g(\theta)
\]
to the displacement produced by the \emph{true} continuous gradient flow of $\mathcal{L}$ over elapsed time $\eta$,
\[
    \Delta\theta_{\text{flow}} \;=\; \varphi_\eta(\theta) - \theta,
    \qquad
    \varphi_t(\theta) \;:=\; \text{solution at time } t \text{ of } \dot\theta(s) = -g(\theta(s)),\ \theta(0) = \theta.
\]
A Taylor expansion of $\varphi_\eta$ gives $\Delta\theta_{\text{flow}} = -\eta g + (\eta^2/2)\,Hg + O(\eta^3)$, so the per-unit-time discrepancy
\[
    \mathcal{T}(\theta, \eta) \;:=\; \frac{\Delta\theta_{\text{flow}} - \Delta\theta_{\text{GD}}}{\eta} \;=\; \frac{\eta}{2}\,H(\theta)\,g(\theta) \;+\; O(\eta^2)
\]
matches $\nabla\mathcal{R}(\theta,\lambda)$ exactly when $\lambda = \eta p/4$. Given observations $\{(H(\theta_t)g(\theta_t),\,\mathcal{T}(\theta_t,\eta))\}_{t=1}^{T}$ at parameter values $\theta_t$ visited along a probe trajectory, the scalar least-squares fit admits a closed form:
\begin{equation}
    \widehat{\tfrac{\lambda}{p}} \;=\; \arg\min_{c}\; \sum_{t=1}^{T} \bigl\lVert \mathcal{T}(\theta_t,\eta) - 2c\, H(\theta_t)g(\theta_t) \bigr\rVert_2^2
    \;=\; \tfrac{1}{2}\,\frac{\sum_t \bigl\langle H(\theta_t)g(\theta_t),\, \mathcal{T}(\theta_t,\eta) \bigr\rangle}{\sum_t \bigl\lVert H(\theta_t)g(\theta_t) \bigr\rVert_2^2}.
    \label{eq:igr-ols}
\end{equation}

\paragraph{Computing the reference flow.}
We compute $\Delta\theta_{\text{flow}}$ with classical fourth-order Runge--Kutta integration of $\dot\theta=-g(\theta)$, rather than by taking smaller Euler steps. Sub-stepped Euler would introduce the same first-order backward-error term we are trying to estimate, so it would partially build the target into the numerical reference. RK4 gives a high-accuracy approximation to the continuous gradient flow while preserving the interpretation of $\mathcal{T}$ as the discrepancy between one discrete GD step and the corresponding continuous-time flow segment.

\paragraph{Experimental protocol for Figure \ref{fig:igr}.}
We train five-hidden-layer MLPs on MNIST with hidden width $w \in \{50, 100, 200, 400, 800, 1600\}$ and step size $\eta \in \{5\times10^{-4}, 10^{-3}, 5\times10^{-3}, 10^{-2}, 5\times10^{-2}, 10^{-1}, 5\times10^{-1}\}$. We use $\tanh$ activations rather than ReLU activations because the estimator for this regularizer requires $H(\theta)g(\theta)$; equivalently, it differentiates the gradient of the gradient-penalty regularizer, so the loss must be twice differentiable in the parameters. Training uses mini-batch stochastic gradient descent with batch size $128$ on a $10{,}000$-sample MNIST training subset, with a $1{,}000$-sample held-out validation subset (drawn disjointly from the MNIST training split) and a $5{,}000$-sample test subset, with no weight decay, dropout, momentum, or learning-rate schedule.

We evaluate each model at the checkpoint with highest \emph{validation} accuracy among checkpoints that fit the training data. We use validation rather than test accuracy for checkpoint selection so that the test accuracy reported in Figure~\ref{fig:igr} (Middle) is independent of the selection criterion. Because the smooth $\tanh$ networks and reduced training subset do not always reach exact interpolation within the computational budget, we use a relaxed train-accuracy threshold of $90\%$ and exclude runs that fail to reach it. At the selected snapshot $\theta^\star$, we compute $R_{IG}(\theta^\star) = \|g(\theta^\star)\|_2^2 / p$ and estimate $\lambda$ by running $T=5$ additional full-batch probe steps from $\theta^\star$, collecting $(H(\theta_t)g(\theta_t),\,\mathcal{T}(\theta_t,\eta))$ at each step. The Hessian-vector product is computed by double backpropagation, and each RK4 reference uses $k=10$ substeps. For computational tractability, the probe loss for $R_{IG}$ and $\hat\lambda$ is evaluated on a fixed $2{,}048$-example training subset.

\paragraph{Caveats on the per-step estimator.}
The identity $\mathcal{T} = (\eta/2)\,Hg + O(\eta^2)$ is a local Taylor/backward-error prediction under smoothness of $\mathcal{L}$. Thus the recovery panel in Figure \ref{fig:igr} should be read as a calibration check for our trajectory estimator against a known analytic coefficient, not as an independent proof that IGR explains test accuracy. The left and middle panels retain the qualitative Barrett Figure 2 comparisons between $\hat\lambda$, $R_{IG}$, and test accuracy; the right panel is the direct methodological check, plotting $\hat\lambda$ against $\eta p/4$.


\section{Compute resources}
\label{app:compute}

We summarize the computational resources used for the experiments below.
Experimental protocols, including data generation, model architectures,
optimization, and stopping criteria, are described in
Appendix~\ref{app:ols-methods} and the other experiment-specific sections.
The times reported here are allocated compute budgets, not measured run
times or minimum hardware requirements. Code and instructions for running
the experiments are available in the accompanying GitHub repository.

\paragraph{Hardware.}
Experiments used CPU servers with 64 x86\_64 cores and 1\,TB of RAM, with
NVIDIA L40S GPUs for neural-network training. Each GPU training job used a
single GPU; the per-job CPU and host-memory allocations are given below.
The full server resources were shared across jobs.

\paragraph{Linear regression.}
The OLS early-stopping and bootstrap experiments used 10-dimensional linear
regression and ran on CPUs. Matrix-recovery and bootstrap-recovery jobs
were allocated 2 CPUs, 8\,GB of RAM, and one hour each. The
\(\lambda\)-versus-epochs experiment was allocated 2 CPUs, 16\,GB of RAM,
and four hours. These experiments include single-endpoint recovery at
\(\sigma=1\) and \(\sigma=10\), fixed-iterate and variable-stop recovery
with 100 endpoints in each of five pools, and bootstrap recovery with
100 resamples in each of ten pools, together with the noise-level sweep
in Appendix~\ref{app:ols-methods}.

\paragraph{Implicit gradient regularization.}
The experiments in Appendix~\ref{app:igr-methods} used one GPU, 4 CPUs,
and 32\,GB of RAM per job. The width and step-size sweep comprised six
widths and seven step sizes, with 50 epochs of mini-batch SGD on the MNIST
subset. Each setting also used five full-batch probe steps, with ten
fourth-order Runge--Kutta (RK4) substeps per probe; the sweep was allocated
four hours. The four longer trajectory experiments used synthetic data
at \(\eta\in\{0.01,0.03\}\) for 50 gradient-descent steps, and MNIST with
\(\tanh\) and ReLU activations at \(\eta=0.01\) for 30 steps. Each was
allocated six hours to accommodate Hessian--gradient products and the RK4
reference flow used by the trajectory-deviation estimator.

\paragraph{Elastic-net recovery.}
The elastic-net experiment comprised a \(6\times6\) grid of
\((\lambda_1,\lambda_2)\) values with ten seeds, for 360 runs in total.
Runs were distributed across eight parallel workers, each allocated one
GPU, 4 CPUs, 32\,GB of RAM, and four hours to process its assigned runs.

\paragraph{Dropout and ridge replacement.}
Figure~\ref{fig:dropout-l2} uses 480 models: two depths, three widths,
eight dropout rates, and ten seeds. Each model was trained with one GPU,
2 CPUs, and 24\,GB of RAM, with a four-hour allocation. Seven pilot runs
were also extended to assess the robustness of the estimated coefficient.
The training-duration analysis in Figure~\ref{fig:dropout-checkpoint-stability}
uses estimates at checkpoints from the same training runs and requires no
additional training. The matched replacement experiment
(Appendix~\ref{app:dropout-replacement}) added 20 runs: for each of ten
seeds, one model trained with the fitted ridge penalty and one without
that penalty. Each replacement run was allocated one GPU, 2 CPUs,
24\,GB of RAM, and one hour.

\section{Scaling to larger and more complex systems}
\label{app:scaling}

Our experiments are a first step toward studying implicit regularization in contemporary deep learning systems. Applying the framework to CNNs or large language models requires care in choosing the candidate regularizer family and isolating the effects of different training mechanisms. Each new architecture or training technique could be a substantial study in its own right. Here we describe the challenges and considerations for extending this work.

\paragraph{Model size and estimation.}
The size of a model, strictly speaking, is not the only limiting factor. Our approach regresses the loss gradient, a $p$-dimensional vector, on the gradients of $r$ candidate regularizers. For a fixed family, increasing $p$ can provide more estimating equations; the relevant statistical quantities are the dimension $r$, the conditioning of the candidate gradients, and the effective parameter count $p_{\mathrm{eff}}$ described in Appendix~\ref{app:pac}. Correlated errors can limit the information gained from additional parameters. Thus parameter count alone does not determine how much training data or how many endpoints are needed. Larger models also increase the cost of computing gradients, and obtaining additional training endpoints can be expensive, even though loss gradients can be computed without storing the full Jacobian (Appendix~\ref{app:computing-grads}).

\paragraph{Explicit, implicit, and interacting effects.}
Explicit regularization enters the training objective directly; implicit regularization arises through architectural or algorithmic choices that affect the updates or selected solution. Our estimates describe an effective regularizer conditional on the training configuration. Attributing this regularizer to a particular component requires controlled ablations, since its effect may depend on other components. Such interactions can themselves be modeled by including candidate regularizers representing joint effects and fitting them across configurations in which the relevant components are varied.

\paragraph{Disentangling modern training procedures.}
Large language models may involve multiple datasets and varying objectives during pre-training and post-training, in addition to explicit regularization and many architecture and optimizer modifications. Studying these effects requires accounting for known explicit penalties, isolating a data regime and objective, and conducting ablation studies on mechanisms that may induce implicit regularization. Similar considerations apply to CNNs and vision transformers: the challenge is to distinguish the effects of the components used together during training. Access to model weights alone does not necessarily provide the training data, objective, or optimizer state needed for these comparisons. The degree of access to the training environment therefore affects which questions can be studied with an open-weight model.

\paragraph{Choosing and transferring regularizers.}
Theoretical work remains important for proposing mathematical models that capture the effects of training procedures. There is nuance and skill in finding simple descriptions of complex phenomena, and demonstrating a novel implicit bias requires both an informative candidate family and experiments that isolate the effect of interest. Transferring an estimated bias between architectures poses a further challenge: a regularizer expressed in terms of weights requires a clear map between their parameterizations. Architectural choices such as parameter-free skip connections offer a more tractable setting because models with and without the mechanism can share the same weight coordinates. This could allow controlled experiments asking whether an explicit regularizer reproduces the effects of the architectural choice. Developing and validating such extensions remains future work.

\newpage
\section*{NeurIPS Paper Checklist}

\begin{enumerate}

\item {\bf Claims}
    \item[] Question: Do the main claims made in the abstract and introduction accurately reflect the paper's contributions and scope?
    \item[] Answer: \answerYes{} 
    \item[] Justification: The abstract and introduction claim an empirical gradient-matching framework for estimating implicit regularization, with validation on explicit penalties, early stopping, dropout, and step-size-induced regularization. These claims are matched by the framework in Section~2, the experiments in Section~3, and the conclusions in Section~4.

\item {\bf Limitations}
    \item[] Question: Does the paper discuss the limitations of the work performed by the authors?
    \item[] Answer: \answerYes{}
    \item[] Justification: Section~4 discusses limitations including effective-sample-size constraints, collinearity and identifiability of regularizer gradients, difficulty distinguishing regularizers with similar gradients, and cases where an implicit effect may not be sensibly replaced by an explicit regularizer.

\item {\bf Theory assumptions and proofs}
    \item[] Question: For each theoretical result, does the paper provide the full set of assumptions and a complete (and correct) proof?
    \item[] Answer: \answerYes{} 
    \item[] Justification: The assumptions behind the gradient-matching framework are stated in Section~2, and the formal results are accompanied by assumptions and proofs in Appendices~A--C. 

    \item {\bf Experimental result reproducibility}
    \item[] Question: Does the paper fully disclose all the information needed to reproduce the main experimental results of the paper to the extent that it affects the main claims and/or conclusions of the paper (regardless of whether the code and data are provided or not)?
    \item[] Answer: \answerYes{}
    \item[] Justification: The paper gives the estimator equations and Appendix~\ref{app:experiment-details} provides the main experimental protocols, including data-generating processes, MNIST splits and subsets, architectures, optimizers, learning rates, batch sizes, epoch limits, early-stopping rules, checkpoint choices, and seed counts. The implementation and reproduction workflow are publicly available at \url{https://github.com/jrudoler/estimating-implicit-regularization}; see the code and data availability paragraph in Appendix~\ref{app:experiment-details}.

\item {\bf Open access to data and code}
    \item[] Question: Does the paper provide open access to the data and code, with sufficient instructions to faithfully reproduce the main experimental results, as described in supplemental material?
    \item[] Answer: \answerYes{}
    \item[] Justification: The public repository at \url{https://github.com/jrudoler/estimating-implicit-regularization} includes the implementation, dependency/workflow information, and scripts for reproducing the main experimental results. The experiments use synthetic data generated by the scripts and MNIST, with data access and preprocessing described in Appendix~\ref{app:experiment-details} and the repository documentation.

\item {\bf Experimental setting/details}
    \item[] Question: Does the paper specify all the training and test details (e.g., data splits, hyperparameters, how they were chosen, type of optimizer) necessary to understand the results?
    \item[] Answer: \answerYes{} 
    \item[] Justification: Appendix~\ref{app:experiment-details} specifies the training and test details for the main figures, including synthetic-data distributions, train/test splits, MNIST subsets, model architectures, dropout/depth/width grids, optimizers, learning rates, batch sizes, epoch limits, early-stopping criteria, and estimator-fitting details.

\item {\bf Experiment statistical significance}
    \item[] Question: Does the paper report error bars suitably and correctly defined or other appropriate information about the statistical significance of the experiments?
    \item[] Answer: \answerYes{}
    \item[] Justification: The experiments report uncertainty or replicate-level variability: Fig.~3 uses means with $\pm 1$ standard-error bars across seeds, Fig.~4D reports a 95\% confidence band, and Fig.~5 plots individual-seed estimates with 10 seeds per configuration. The paper does not use formal hypothesis tests, since several experiments are calibration or recovery checks against known explicit or analytic targets.

\item {\bf Experiments compute resources}
    \item[] Question: For each experiment, does the paper provide sufficient information on the computer resources (type of compute workers, memory, time of execution) needed to reproduce the experiments?
    \item[] Answer: \answerYes{}
    \item[] Justification: The paper reports model sizes, batch sizes, epoch limits, and dataset/subset sizes. Appendix \ref{app:compute} discusses the compute resources used for the experiments, which were run on an internal cluster with GPU. 
    
\item {\bf Code of ethics}
    \item[] Question: Does the research conducted in the paper conform, in every respect, with the NeurIPS Code of Ethics \url{https://neurips.cc/public/EthicsGuidelines}?
    \item[] Answer: \answerYes{}
    \item[] Justification: The research described is methodological and uses synthetic data and MNIST experiments; it does not describe human-subject data, private or sensitive data, or high-risk model or dataset releases.

\item {\bf Broader impacts}
    \item[] Question: Does the paper discuss both potential positive societal impacts and negative societal impacts of the work performed?
    \item[] Answer: \answerNA{}
    \item[] Justification: The paper discusses positive scientific and practical impacts, such as giving practitioners a cheaper diagnostic for understanding training configurations as compute budgets scale. It does not however discuss negative societal impacts or misuse scenarios -- the method introduced by the paper is itself a tool for furthering understanding of neural networks and their inner workings. As such it has potential to accelerate the development of neural networks and thereby shares the many risks of advancing AI, but none of those are in any way specific to our work. 
    
\item {\bf Safeguards}
    \item[] Question: Does the paper describe safeguards that have been put in place for responsible release of data or models that have a high risk for misuse (e.g., pre-trained language models, image generators, or scraped datasets)?
    \item[] Answer: \answerNA{}
    \item[] Justification: The paper does not release high-risk models or datasets; the experiments use synthetic data and MNIST, and the contribution is a methodological estimator rather than a deployable model or scraped dataset with high misuse risk.

\item {\bf Licenses for existing assets}
    \item[] Question: Are the creators or original owners of assets (e.g., code, data, models), used in the paper, properly credited and are the license and terms of use explicitly mentioned and properly respected?
    \item[] Answer: \answerYes{} 
    \item[] Justification: Appendix~\ref{app:experiment-details} credits MNIST and gives its dataset-card license (MIT), together with the BSD 3-Clause licenses and upstream license URLs for PyTorch and TorchVision. MNIST is downloaded from upstream rather than redistributed as a new dataset.

\item {\bf New assets}
    \item[] Question: Are new assets introduced in the paper well documented and is the documentation provided alongside the assets?
    \item[] Answer: \answerYes{}
    \item[] Justification: The released research code is documented at \url{https://github.com/jrudoler/estimating-implicit-regularization}, including environment setup, workflow commands, and experiment descriptions. Appendix~\ref{app:experiment-details} links to the release. No new dataset or model checkpoint is introduced as a separate asset.

\item {\bf Crowdsourcing and research with human subjects}
    \item[] Question: For crowdsourcing experiments and research with human subjects, does the paper include the full text of instructions given to participants and screenshots, if applicable, as well as details about compensation (if any)? 
    \item[] Answer: \answerNA{}
    \item[] Justification: The paper does not include crowdsourcing experiments or research with human subjects.

\item {\bf Institutional review board (IRB) approvals or equivalent for research with human subjects}
    \item[] Question: Does the paper describe potential risks incurred by study participants, whether such risks were disclosed to the subjects, and whether Institutional Review Board (IRB) approvals (or an equivalent approval/review based on the requirements of your country or institution) were obtained?
    \item[] Answer: \answerNA{}
    \item[] Justification: The paper does not include human-subject research, so IRB approval or equivalent review is not applicable.

\item {\bf Declaration of LLM usage}
    \item[] Question: Does the paper describe the usage of LLMs if it is an important, original, or non-standard component of the core methods in this research? Note that if the LLM is used only for writing, editing, or formatting purposes and does \emph{not} impact the core methodology, scientific rigor, or originality of the research, declaration is not required.
    \item[] Answer: \answerNA{}
    \item[] Justification: The core method development and experiments do not involve LLMs as an important, original, or non-standard component.

\end{enumerate}

\end{document}